%% file: main.tex
\documentclass[11pt]{article}
\usepackage{amsmath,amssymb,amsfonts}
\usepackage{amsthm}
\usepackage[]{algorithm,algorithmic} 
\usepackage{mathtools}
\usepackage{hyperref}
\usepackage[sort&compress]{natbib}

\usepackage{bbm}
\usepackage[margin=1in]{geometry}
\usepackage{caption}
\usepackage{subcaption}
\usepackage{graphicx}

\newcommand{\real}{\mathbb{R}}

\newcommand{\RR}{\mathbb{R}}

\mathchardef\mhyphen="2D 

\DeclarePairedDelimiter{\norm}{\lVert}{\rVert}
\DeclarePairedDelimiter{\abs}{\lvert}{\rvert}

\DeclareMathOperator{\rank}{rank}

\DeclareMathOperator{\diag}{diag}

\newtheorem{thm}{Theorem}[section]
\newtheorem{lemma}[thm]{Lemma}

\newtheorem{definition}[thm]{Definition}

\newtheorem{lem}[thm]{Lemma}

\theoremstyle{definition}

\theoremstyle{remark}

\usepackage{titlesec} 

\newcommand{\Amap}{\mathcal{A}}

\newcommand{\dm}{d}
\newcommand{\dmone}{d_1}
\newcommand{\dmtwo}{d_2}
\newcommand{\ncons}{m}
\newcommand{\overk}{k}

\newcommand{\inprod}[2]{\langle #1, #2 \rangle}
\newcommand{\twonorm}[1]{\left\|#1\right\|_2}
\newcommand{\twoinfnorm}[1]{\left\|#1\right\|_{2,\infty}}

\newcommand{\infnorm}[1]{\left\|#1\right\|_{\infty}}
\newcommand{\fronorm}[1]{\left\|#1\right\|_{\mbox{\tiny{F}}}}
\newcommand{\opnorm}[1]{\left\|#1\right\|_{\mbox{\tiny{\textup{op}}}}}
\newcommand{\nucnorm}[1]{\left\|#1\right\|_*}
\newcommand{\oneonenorm}[1]{\left\|#1\right\|_{1,1}}

\DeclareMathOperator{\dist}{dist}

\newcommand{\tr}{\mathop{\rm tr}}

\newcommand{\str}{\normalfont{\textrm{str}}}

\let\originalleft\left
\let\originalright\right
\renewcommand{\left}{\mathopen{}\mathclose\bgroup\originalleft}
\renewcommand{\right}{\aftergroup\egroup\originalright}

\usepackage{multirow}

\newcommand{\trux}{x_\natural}
\newcommand{\truM}{M_\natural}
\newcommand{\trur}{r_\natural}
\newcommand{\truU}{U_\natural}
\newcommand{\truV}{V_\natural}

\newcommand{\truS}{S_\natural}
\newcommand{\truSig}{\Sigma_\natural}
\newcommand{\ones}{\mathbf{1}}
\newcommand{\Hessian}{\mathcal{H}}
\usepackage[utf8]{inputenc}
\usepackage[english]{babel}
\usepackage{xcolor}

\title{Flat minima generalize for low-rank matrix recovery}

\author{
Lijun Ding\thanks{Wisconsin Institute for Discovery, University of Wisconsin - Madison, Madison, WI 53715; \texttt{lding47@wisc.edu}. Research supported by NSF CCF-2023166.}
\and
Dmitriy Drusvyatskiy\thanks{Department of Mathematics, University of Washington, Seattle, WA 98195; \texttt{www.math.washington.edu/$\sim$ddrusv}.
Research supported by NSF CAREER DMS-1651851 and CCF-2023166.}
\and 
Maryam Fazel\thanks{Department of Electrical \& Computer Engineering, University of Washington, Seattle, WA 98195; \texttt{people.ece.uw.edu/fazel\_maryam/}. Research supported by NSF CCF-2023166, CCF-2007036, and DMS-1839371. }
\and 
Zaid Harchaoui \thanks{Department of Statistics, University of Washington, Seattle, WA 98195; \texttt{https://faculty.washington.edu/zaid}.
}
}
\date{}

\begin{document}

\maketitle

\begin{abstract}
	Empirical evidence suggests that for a variety of overparameterized nonlinear models, most notably in neural network training, the growth of the loss around a minimizer strongly impacts its performance. Flat minima---those around which the loss grows slowly---appear to generalize well. This work takes a step towards understanding this phenomenon by focusing on the simplest class of overparameterized nonlinear models: those arising in low-rank matrix recovery. We analyze overparameterized matrix and bilinear sensing, robust PCA, covariance matrix estimation, and single hidden layer neural networks with quadratic activation functions. In all cases, we show that flat minima, measured by the trace of the Hessian, {\em exactly recover} the ground truth under standard statistical assumptions. For matrix completion, we establish weak recovery, although empirical evidence suggests exact recovery holds here as well. We conclude with synthetic experiments that illustrate our findings and discuss the effect of depth on flat solutions. 
\end{abstract}

\section{Introduction}\label{sec: intro}
Recent advances in machine learning and artificial intelligence have relied on fitting highly overparameterized models, notably deep neural networks, to observed data \cite{tan2019efficientnet,kolesnikov2020big,huang2019gpipe,zhang2021understanding}. In such settings, the number of parameters of the model is much greater than the number of data samples, thereby resulting in models that achieve near-zero training error. Although classical learning paradigms caution against overfitting, recent work suggests ubiquity of the ``double descent'' phenomenon \cite{belkin2019reconciling}, wherein significant overparameterization actually improves generalization. There is an important caveat, however, that is worth emphasizing. There is typically a continuum of models with zero training error; some of these models generalize well and some do not. Reassuringly, there is evidence that basic algorithms, such as the stochastic gradient method, are implicitly biased towards finding models that do generalize; see for example \cite{soudry2018implicit,gunasekar2018implicitNN,jacot2018neural,heckel2020compressive,jastrzkebski2017three,smith2017bayesian,hoffer2017train,masters2018revisiting,neyshabur2014search,gunasekar2018implicit,du2018algorithmic,mulayoff2020unique}. 
Other seminal works \cite{bartlett1998sample,bartlett2002rademacher,neyshabur2017exploring} seeking to explain generalization  have focused on quantifying  stability, capacity, and margin bounds. 
Understanding generalization of overparameterized models remains an active  area of research, and is the topic of our work.

Existing literature highlights two intriguing properties---{\bf small norm}  and {\bf flat landscape}---that  correlate with generalization \cite{neyshabur2017exploring,dziugaite2017computing,dinh2017sharp}. 
Indeed, it has long been known that the magnitude of the weights plays an important role for neural network training. As a result, one typically incorporates a squared $\ell_2$-penalty on the weights---called weight decay---when applying iterative methods. One intuitive explanation is that minimizing the square Frobenius norm of the factors in matrix factorization problems is equivalent to minimizing the nuclear norm---a well-known regularizer for inducing low-rank structure \cite{recht2010guaranteed}. Far reaching generalizations of this phenomenon for various neural network architectures have been recently pursued in \cite{savarese2019infinite,ongie2019function,ongie2022role}. In parallel, empirical evidence \cite{hochreiter1997flat,Keskar2016,li2017visualizing} strongly suggests that those models around which the landscape is flat---meaning the training loss grows slowly---generalize well. See Figure~\ref{fig:graphicalflatsolution} for an illustration of flat and sharp minima. Inspired by this observation, a variety of algorithms have been proposed to explicitly bias the iterates towards flat solutions \cite{chaudhari2019entropy,izmailov2018averaging,norton2021diametrical,foret2020sharpness}, with impressive observed performance. In contrast to the magnitude of the weights, the theoretic basis for flatness is much less clear even for simple overparameterized nonlinear problems. The goal of our work is to answer the following question:
\begin{quote}
	\centering
	{\em Do flat minimizers generalize for a broad family of overparameterized problems?}
\end{quote} 

\begin{figure}[h]
	\centering
	\includegraphics[width= 0.4 \textwidth]{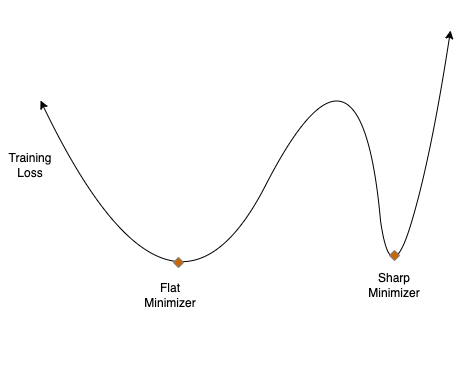}
	\caption{Flat vs. sharp minima of the training loss. }
	\label{fig:graphicalflatsolution}
\end{figure}

Putting generalization aside, one would hope that flat solutions are in some sense regular, occurring in a benign region where algorithms perform well. For example, numerical methods for neural network training are strongly influenced by how {\bf balanced} the parameters appear. Namely, the set of interpolating neural networks contains models with consecutive weight matrices that are poorly scaled relative to each other \cite{du2018algorithmic,shamir2018resnets}.  It has recently been shown that gradient descent in continuous time keeps the factors  balanced \cite{ye2021global,ma2021beyond} for matrix factorization and for deep learning \cite{du2018algorithmic,mulayoff2020unique}. Despite ubiquity of the three notions discussed so far---small norm, flatness, and balancedness---the exact relationship between them is unclear. Thus our secondary question is as follow:

\begin{center}
	{\em Are flat minimizers  nearly norm-minimal and nearly balanced\\ 
		for a broad family of  overparameterized problems?}
\end{center}

\subsection{Problem setting: overparameterized matrix factorization}
We answer both questions in the setting of low-rank matrix factorization---a prototypical  problem class often used to gain insight into more general deep learning models \cite{li2018algorithmic,du2018algorithmic,ye2021global}. 
Setting the stage, consider a ground truth matrix $\truM\in \real^{\dmone\times \dmtwo}$ with rank $\trur$.  The goal is to recover $\truM$ from the observed measurements $b=\Amap(\truM)$ under a linear measurement map $\Amap\colon\real^{\dmone \times \dmtwo} \rightarrow \real^{\ncons}$.
A common approach to this task is through the nonconvex optimization problem:
\begin{equation}\label{eq: main minimization}
	\min_{L,R} \quad f(L,R):=\twonorm{\Amap(LR ^\top) -b}^2\qquad \textrm{with }L \in \real^{\dmone \times \overk}\textrm{ and }R \in \real^{\dmtwo \times \overk}.
\end{equation}
The set of minimizers of $f$, which we denote by $\mathcal{S}$, consists of all solutions to the equation $\Amap(LR ^\top) =b$.
In order to model overparameterization, we focus on the rank-overparameterized setting $k\geq \trur$; indeed $k$ can be arbitrarily large.
The three notions discussed so for can be formally defined for pairs $(L,R) \in \mathcal{S}$ as follows.
\begin{itemize}
	\item $(L,R)$ is {\bf norm-minimal}  if it minimizes  over $\mathcal{S}$ the square Frobenius norm $\fronorm{L}^2+\fronorm{R}^2$.
	\item $(L,R)$ is {\bf balanced} if it satisfies $L^{\top}L=R^{\top}R$.
	\item $(L,R)$ is {\bf flat}  if it minimizes over $\mathcal{S}$ the ``scaled trace'' of the Hessian, $\str(D^2 f(L,R))$.
\end{itemize}
Thus being norm-minimal  means that $(L,R)$ is the closest pair from $\mathcal{S}$  to the origin in Frobenius norm. Being balanced amounts to requiring $L$ and $R$ to have the same singular values and right-singular vectors. Flat solutions are defined in terms of the ``scaled trace'' of the bilinear form $D^2 f(L,R)$ defined as
\begin{equation}\label{eqn:sctrace}
	\str(D^2 f(L,R)) := \tfrac{1}{d_1}\sum_{i\leq d_1,j\in [k]} D^2 f(L,R)[(e_ie_j^{\top},0_{d_2\times k})]+\tfrac{1}{d_2}\sum_{i>d_1,j\in[k]}D^2 f(L,R)[(0_{d_1\times k},e_ie_j^{\top})].
\end{equation}
where $e_i$ and $e_j$ are the unit coordinate vectors in $\RR^{d_1+d_2}$ and $\RR^k$, respectively. 
%
%
%
%
%
 In the square setting $d_1=d_2=d$, the scaled trace reduces to the usual trace divided by $d$.
The scaled trace appears to have not been used previously in the literature, but is important in order to account for a possible mismatch in the dimension of the $L$ and $R$ factors. A number of recent papers use the trace of the Hessian to measure flatness (e.g. \cite{dinh2017sharp}). Other alternatives are possible, such as the maximal eigenvalue \cite{dinh2017sharp,mulayoff2020unique} or the condition number \cite{liu2021noisy}, but we do not focus on them here.
Our main contribution can be succinctly summarized as follows:

\begin{center}
	{\em For various statistical models,  flat solutions  of \eqref{eq: main minimization}
		\underline{exactly recover} $\truM$. \\
		Moreover, flat solutions have nearly minimal norm and are almost balanced.}
\end{center}

The exact recovery guarantee may be striking at first because flat solutions are distinct from  minimal norm solutions, and thus do not correspond to nuclear norm minimization over $\mathcal{S}$. Yet, our main result shows that flat solutions do exactly recover the ground truth  $\truM$ under standard statistical assumptions.
The precise statistical models for which this is the case are matrix and bilinear sensing,  robust PCA (or PCA with outliers), covariance matrix estimation, and single hidden layer neural networks with quadratic activation functions.
Moreover, we prove weak recovery for the matrix completion problem, though our numerical experiments suggest that exact recovery holds here as well.

\subsection{Main results and outline of the paper}
We next outline our main results and the arguments that underpin them. We begin in Section~\ref{sec: fbminimalell2norm} with the idealized ``population level'' setting where $\mathcal{A}$ is the identity map. In this case, we show that there is no distinction between flat, norm-minimal, and balanced solutions. As soon as $\mathcal{A}$ deviates from the identity, however, all three notions become distinct in general.

\begin{figure}[h]
	\centering
	\includegraphics[scale=0.5]{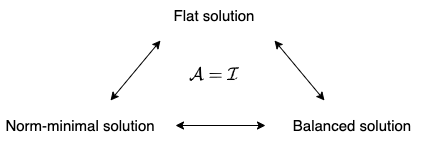}
	\caption{Equivalence between balanced, minimal norm, and flat solutions when $\mathcal{A}=\mathcal{I}$.}
	\label{fig:equiv_noa}
\end{figure}

An immediate difficulty with analyzing flat solutions of the problem \eqref{eq: main minimization} with a general measurement map $\mathcal{A}$ is that flat solutions are defined as minimizers of a highly nonconvex optimization problem corresponding to minimizing the scaled trace over the solution set. In Section~\ref{sec:conv_relax_flat}, we derive a simple convex relaxation of flat minimizers. Setting the notation, let us write $\mathcal{A}$ as $\mathcal{A}(X)=(\langle A_i,X\rangle,\ldots, \langle A_m,X\rangle)$ for some matrices $A_i\in\RR^{d_1\times d_2}$ and define the
``rescaling'' matrices
\begin{equation}
	D_1 := \left(\frac{1}{\ncons \dmtwo} \sum_{i=1}^m A_i A_i^\top \right)^{\frac{1}{2}}\quad \text{and}\quad D_2 := \left(\frac{1}{\ncons \dmone} \sum_{i=1}^m A_i^\top  A_i\right)^{\frac{1}{2}}.
\end{equation} 
We will show in Theorem~\ref{thm: equivalenceConvexNonconvexScaledTraceHessioan} that flat solutions can be identified with minimizers of the problem 
\begin{equation}\label{eq: PerturbnucnormXrank}
	\min_{X\in \RR^{d_1\times d_2}:~\rank(X)\leq k} \nucnorm{D_1XD_2}\qquad\textrm{subject to}\qquad \Amap(X)=b. 
\end{equation}
It is worthwhile to note that without the $D_1$ and $D_2$	matrices and without the rank constraint, the problem \eqref{eq: PerturbnucnormXrank} is classically known  to characterize norm-minimal solutions and is known as nuclear norm minimization. Herein, we already see the distinction between the two solution concepts.  A natural convex relaxation for  flat solutions simply drops the rank constraint: 
\begin{equation}\label{eq: PerturbnucnormX_con_relax_intro}
	\min_{X\in \RR^{d_1\times d_2}}~ \nucnorm{D_1XD_2}\qquad\textrm{subject to}\qquad \Amap(X)=b. 
\end{equation}
Summarizing, verifying that flat solutions exactly recover $\truM$ is  reduced to showing that  $\truM$ (which has rank $\trur$) is the unique solution of the convex problem \eqref{eq: PerturbnucnormX_con_relax_intro}. 

In Section~\ref{sec:RIP}, we will show that if the linear map $\mathcal{A}$ satisfies $\ell_2/\ell_2$ or $\ell_1/\ell_2$  restricted isometry properties (RIP) and the rescaling matrices $D_1$ and $D_2$ are sufficiently close to the identity, then $\truM$ is the unique solution of \eqref{eq: PerturbnucnormX_con_relax_intro}. As a consequence, we deduce that flat solutions exactly recover $\truM$ for matrix sensing \cite{recht2010guaranteed,candes2011tight} and bilinear sensing \cite{ling2015self,ahmed2013blind} problems with Gaussian design. The former corresponds to the setting where the entries of $A_i$ are independent standard Gaussian random variables, while the latter corresponds to the setting $A_i=a_ib_i^{\top}$ where $a_i\in\RR^{d_1}$ and $b_i\in \RR^{d_2}$ are independent standard Gaussian  vectors.  The end result is the following theorem. Simplifying notation, we set $d_{\max} =\max\{d_1,d_2\}$ and $d_{\min} = \min \{d_1,d_2\}$.

\begin{thm}[Matrix and bilinear sensing (Informal)]
	Suppose that  $\mathcal{A}$ is generated according to a Gaussian matrix sensing or bilinear sensing model. Then as long as we are in the regime  $m\gtrsim \trur d_{\max}$ and $d_{\min}\gtrsim \log \ncons$, with high probability, any flat solution $(L_f,R_f)$ satisfies $L_fR_f^\top = \truM,
	$ and is nearly norm-minimal and nearly balanced.
\end{thm}
Note that our requirement on the sample size $m\gtrsim \trur \dm_{\max}$ matches the known regime for exact recovery with nuclear norm minimization \cite{candes2011tight,cai2015rop}. 
Since we are interested in the high dimensional regime, the extra condition $\dm_{\min}\gtrsim \log (m)$ can be assumed without harm. Appendix \ref{sec: noisyobservation} presents a generalization of this result when the measurements $b$ are corrupted by noise.

We next move on to analyzing the matrix completion problem in Section~\ref{sec:mat_comp}. We focus on the Bernoulli model, wherein each matrix $A_i$ takes the form $A_i= \xi_{ij} e_i e_j^{\top}$, where $e_i$ and $e_j$ denote the $i$'th and $j$'th coordinate vectors in $\RR^d$ and $\xi_{ij}$ are independent Bernoulli random variables with success probability $p\in (0,1)$. The main difficulty with analyzing the matrix completion problem is that the linear map $\mathcal{A}$ does not have good restricted isometry properties. Moreover, the existing techniques for analyzing the nuclear norm relaxation of the matrix completion problem  \cite{recht2011simpler,candes2009exact} do not directly apply to the problem \eqref{eq: PerturbnucnormX_con_relax_intro} because of the dependence between the rescaling matrices $D_1$, $D_2$, and the observation map $\mathcal{A}$. Consequently, we settle for an approximate recovery guarantee. 
\begin{thm}[Matrix completion (Informal)]
	Suppose that $\mathcal{A}$ is generated from the Bernoulli matrix completion model with success probability $p>0$ and let $\mu>0$ be the incoherence parameter of  $\truM$.\footnote{See \eqref{eq: incoherenceMu} for the definition of the incoherence parameter $\mu$.}
	Then provided we are in the regime 
	$p\gtrsim \frac{1}{\gamma}\sqrt{\frac{\trur\log(d_{\max})}{d_{\min}}}$,
	with high probability, any flat solution $(L_f,R_f)$ 	
	satisfies	$\nucnorm{L_f R_f^\top -\truM} \leq \gamma\nucnorm{\truM}$ and is nearly norm-minimal and nearly balanced.
\end{thm}
Hence according to this theorem, in order to conclude that flat solutions achieve a constant relative error, we must be in the regime $p\gtrsim \sqrt{\frac{\trur\log d_{\max}}{d_{\min}}}$. This is a stronger 
requirement than is needed for exact recovery of the ground truth matrix by nuclear norm minimization \cite{chen2015incoherence}, which is  $p\gtrsim \mu \trur \log(\mu \trur )\frac{\log (\dm_{\max})}{\dm_{\min}}$. We stress, however, that our numerical results suggest that flat solutions exactly recovery the ground truth matrix in this wider parameter regime.

We next focus on the problem of Robust Principal Component Analysis (PCA) in Section~\ref{sec: extension}. Though this problem is not of the form \eqref{eq: main minimization}, we will see that flat solutions (appropriately defined) exactly recover the ground truth under reasonable assumptions. Specifically, following \cite{candes2011robust,chandrasekaran2011rank}, the robust PCA problem asks to find a low-rank matrix $\truM \in \real^{\dmone\times \dmtwo}$ that has been corrupted by sparse noise $\truS$. Thus, we observe a matrix  $Y\in \real^{\dmone\times \dmtwo}$ of the form 
\begin{equation}\label{eq: PCA}
	Y = \truM + \truS, 
\end{equation}
where the matrix $\truS$ is assumed to have at most $l_\natural$ nonzero entries 
in any column and in any row. 
A popular formulation of the problem (see \cite[Eqn. (19)]{ha2020equivalence}, \cite[Eqn. (6)]{ge2017no}) 
takes the form 
\begin{equation}\label{eqn:key_blah_rpca_intro}
	\min_{L,R}~f(L,R):=\dist^2_{\Omega}(Y-LR^{\top}),
\end{equation}
where $\dist^2_{\Omega}$ is the square Frobenius distance to the sparsity-inducing set $\Omega :=\{ S\mid \oneonenorm{S}\leq \oneonenorm{\truS}\}$. The objective function $f$ is $C^1$-smooth but not $C^2$-smooth. Therefore, in order to measure flatness, we  approximate $f$ near a basepoint $(\tilde L,\tilde S)$ by a certain $C^2$-smooth local model $f_{\tilde{L},\tilde{R}}(L,R)$, introduced in  \cite[Section 4.2]{ha2020equivalence}, \cite[Section 4.3]{ge2017no}. See Section~\ref{sec: extension} for a precise definition of $f_{\tilde{L},\tilde{R}}(L,R)$.
We then define a minimizer of \eqref{eqn:key_blah_rpca_intro} to be {\em flat} if it minimizes the scaled trace $\str(D^2f_{L,R}(L,R))$ over all $(L,R)\in \mathcal{S}$. We will prove the following theorem, which largely follows from the results of \cite{chen2013low}. 

\begin{thm}[Robust PCA (Informal)]
	Let $\mu$ be the strong incoherence parameter of $\truM$.\footnote{ See \eqref{eq: jointIncoherent} for the definition of the strong incoherence parameter $\mu$.}
	Then, in the regime $l_\natural \lesssim \frac{d_{\min}}{\mu \trur}$, any flat minimizer $(L_f,R_f)$  satisfies $L_fR_f^\top = \truM$. 
\end{thm}

Section~\ref{sec: NNQA} analyzes the last problem class of the paper, motivated by the problems of covariance matrix estimation and training of shallow neural networks. Setting the stage, consider a ground truth matrix $\truM$ satisfying 
$$\truM=U_{1,\natural}U_{1,\natural}^{\top}-U_{2,\natural}U_{2,\natural}^{\top},$$ for some matrices $U_{1,\natural}\in \RR^{d\times r_1}$ and $U_{2,\natural}\in \RR^{d\times r_2}$. The goal is to recover 
$\truM$ from the observations
\begin{equation}\label{eqn:gen_cov_intro}
	b_i=x_i^{\top} \truM x_i,
\end{equation}
where $x_1,\ldots, x_m\overset{\text{iid}}{\sim} N(0,I_d)$.  Note that in the special case $r_2= 0$, this problem reduces to covariance matrix estimation  \cite{chen2015exact} and 
further reduces to phase retrieval  when $r_1 = 1$ \cite{candes2013phaselift}. The added generality allows to also model shallow neural networks with quadratic activation functions; see details below.  A natural optimization formulation of the problem takes the form 
\begin{equation}\label{eqn:covar_intro}
	\min_{U_1\in \RR^{d\times k_1},~ U_2\in \RR^{d\times k_2}}~~ f(U_1,U_2)=~\frac{1}{m}\twonorm{\Amap(U_1U_1^\top - U_2U_2^\top -\truM)}^2,
\end{equation}
where the sensing matrices are $A_i=x_ix_i^{\top}$ and $k_i\geq r_i$ for $i=1,2$. Since 
$\str(D^2f(U_1,U_2))) =d \tr(D^2 f(U_1,U_2))$,
we declare a minimizer $(U_{1,f},U_{2,f})$ to be {\em flat} if it has minimal trace 
$\tr(D^2 f(U_1,U_2))$ among all minimizers of  \eqref{eqn:covar_intro}. 
We  prove the following. 

\begin{thm}[Exact recovery]
	In the regime $m\gtrsim C (r_1+r_2)\dm$ and  $\dm \gtrsim C\log m$, with high probability, any flat solution $(U_{f,1},U_{f,2})$ of \eqref{eqn:covar_intro} satisfies
	$U_{f,1}U_{f,1}^\top -
	U_{f,2}U_{f,2}^\top = \truM$.
\end{thm}

Here, our requirement on the sample size $m\gtrsim C (r_1+r_2)\dm$ coincides with the known requirement for exact recovery by nuclear norm minimization \cite{chen2015exact} in terms of $r$ and $\dm$. 
An interesting example of \eqref{eqn:covar_intro} arises from a model of shallow neural networks, analyzed in \cite{soltanolkotabi2018theoretical,li2018algorithmic} for the purpose of studying  energy landscape around saddle points. 
Namely, suppose that given an input vector $x\in \RR^d$ a response vector $y(x)$ is generated by the ``teacher neural network''
$$y(U_{\natural},x) =v^\top q(\truU ^\top x),$$
where the output weight vector $v\in \real^{r}$ has $r_1$ positive entries and $r_2$ negative entries, the 
hidden layer weight matrix $\truU$ has dimensions $\dm \times \trur$, and we use a quadratic activation $q(s) = s^2$ applied coordinate-wise. 
We get to observe a set of $\ncons$  pairs $(x_i,y_i)\in\RR^d\times\RR$, 
where the features $x_i$ are drawn as $x_i\overset{\text{iid}}{\sim} N(0,I_d)$ and the output values are  $y_i=y(x_i)$.
The goal is to fit the data with an overparameterized ``student neural network''
\begin{equation*}
	\hat{y}(U,x)= u^\top q(U^\top x),
\end{equation*}
with hidden weights $U\in \real^{d\times k}$ and output layer weights $u = (\ones_{k_1},-\ones_{k_2})$, where $k_1\geq r_1$, and $k_2\geq r_2$. It is straightforward to see that by partitioning the matrix $U=[U_1,U_2]$, this problem is exactly equivalent to  recovering the matrix $\truM=\truU\diag(v)\truU^\top$ from the observations \eqref{eqn:gen_cov_intro}.

Section \ref{sec: num} numerically validates our theoretical results. Section \ref{sec: Effect_of_depth} summarizes our findings and speculates about the role of depth on generalization properties of flat solutions.

\paragraph{Notation.} Throughout, we let $\RR^{d}$ denote the $d$-dimensional Euclidean space, equipped with the usual  dot-product $\langle x,y\rangle=x^{\top}y$ and the induced Euclidean norm $\|\cdot\|_2$. More generally, the symbol $\|\cdot\|_p$ will denote the $\ell_p$ norm on $\RR^d$. Given two numbers $d_1$ and $d_2$, which will be clear from context, we set $d_{\max}:=\max\{d_1,d_2\}$ and $d_{\min}:=\min\{d_1,d_2\}$. The Euclidean space of $d_1\times d_2$ real matrices $\RR^{d_1\times d_2}$ will always be equipped with the trace inner product $\langle X,Y\rangle=\tr(X^{\top}Y)$ and the induced Frobenius norm $\fronorm{X}=\sqrt{\langle X,X\rangle}$.
The nuclear norm $\nucnorm{X}$ of any matrix $X\in\RR^{d_1\times d_2}$ is the sum of its singular values. We will often use the  characterization of the nuclear norm \cite[Lemma 1]{srebro2005rank}:
\begin{equation}\label{eqn:srebro}
	\nucnorm{X}=\min_{X=LR^{\top}} \fronorm{L}\fronorm{R}=\min_{X=LR^{\top}} \tfrac{1}{2} (\fronorm{L}^2+\fronorm{R}^2).
\end{equation}

\section{Norm-minimal, flat, and balanced solutions with an identity measurement map}\label{sec: fbminimalell2norm}

In this section, we focus on the idealized objective \eqref{eq: main minimization} where the measurement map $\mathcal{A}$ is the identity:
\begin{equation}\label{eqn:pop_level}
	\min_{L\in \RR^{d_1\times k},~ R\in \RR^{d_2\times k}}~f(L,R)=\fronorm{LR^{\top}-M_{\natural}}^2,
\end{equation}
Recall that $M_{\natural}\in \RR^{d_1\times d_2}$ is a rank $\trur$ matrix, $k\geq \trur$ is arbitrary, and the set of minimizers $\mathcal{S}$ of \eqref{eqn:pop_level} coincides with the solution set of the equation $LR^{\top}=M_{\natural}$. We will show in this section that in this setting there is no distinction between norm-minimal, flat, and balanced solutions.  As soon as the measurement map $\mathcal{A}$ is not the identity, the three notions become distinct; this remains true even under standard statistical models as our numerical experiments show. Nonetheless, the simplified setting $\mathcal{A}=\mathcal{I}$ explored in this section will serve as motivation for the rest of the paper.

We begin with the following lemma that provides a convenient expression for $\str(D^2 f(L,R))$.
\begin{lemma}[Scaled trace]\label{lem:flattest}
	The second-order derivative of the function $f$ at any $(L,R)\in \mathcal{S}$ is the quadratic form:
	\begin{equation}\label{eq: HessianAmapIdentity}
		\begin{aligned}
			D^2 f(L,R)[U,V] & = 2\fronorm{LV^\top +UR^\top}^2.
		\end{aligned}
	\end{equation}
	Consequently, the scaled trace is simply
	\begin{equation}\label{eq: scaledtraceAmapIdentity}
		\str(D^2 f(L,R)) = 2 (\fronorm{L}^2+\fronorm{R}^2).
	\end{equation}
\end{lemma}
\begin{proof}
	A straightforward computation shows for any pair $(L,R)$ the expression 
	$$D^2 f(L,R)[U,V]=4\langle LR^{\top}-M_{\natural}, UV^{\top}\rangle+2\fronorm{LV^{\top}+UR^{\top}}^2.$$	
	For pairs $(L,R)\in \mathcal{S}$, the first term on the right is zero yielding the claimed
	expression~\eqref{eq: HessianAmapIdentity}.  To see the expression for the scaled trace, let $e_i\in\RR^{d_1+d_2}$ and $e_j\in\RR^k$ be the $i$'th and $j$'th coordinate vectors. A quick computation shows 
	$D^2 f(L,R)[e_ie_j^{\top}]=2\fronorm{R_j}^2$ for $i\leq d_1$ and $D^2 f(L,R))[e_ie_j^{\top}]=2\fronorm{L_j}^2$ for $i> d_1$. Therefore, from the definition \eqref{eqn:sctrace}, the scaled trace becomes
	$$\str(D^2 f(L,R)) := \tfrac{1}{d_1}\sum_{i\leq d_1}\sum_{j\in [k]} D^2 f(L,R)[e_ie_j^{\top}]+\tfrac{1}{d_2}\sum_{i>d_1}\sum_{j\in[k]} D^2 f(L,R)[e_ie_j^{\top}]=2(\fronorm{L}^2+\fronorm{R}^2),
	$$
	as claimed.
\end{proof}

We are now ready to prove the claimed equivalence between the three properties.

\begin{lemma}[Equivalence]\label{lem: balancedMinimalNorm}  Norm-minimal, flat, and balanced solutions of \eqref{eqn:pop_level} all coincide.
\end{lemma}
\begin{proof}
	First, the equivalence of flat and norm-minimal solutions follows directly from the expression \eqref{eq: scaledtraceAmapIdentity} in Lemma~\ref{lem:flattest}.
	Next, we prove the equivalence between minimal norm and balanced solutions. 
	Suppose $(L,R)\in S$ is balanced.  The equality $L^{\top}L=R^{\top}R$ implies that $L$ and $R$ have the same nonzero singular values and the same set of right singular vectors. Therefore, we may form compact singular value decompositions $L=U_1\Sigma V^\top$ and $R=U_2\Sigma V^\top$.  Since equality $LR^{\top} =\truM$ holds, we see that 
	$U_1\Sigma^2 U_2^\top = \truM$. Hence, the nuclear norm of $\truM$ is simply $\nucnorm{\truM} = \tr(\Sigma^2)$. Noting the equality $\frac{1}{2}\left(\fronorm{L}^2 +\fronorm{R}^2\right) = \tr(\Sigma^2)$ along with \eqref{eqn:srebro}, we deduce that $(L,R)$ is a minimal norm solution, as claimed. 
	Conversely, suppose that $(L,R)$ is a minimal norm solution. Define  the function 
	$$\varphi(B)=\frac{1}{2}\fronorm{LB}^2+\frac{1}{2}\fronorm{RB^{-\top}}^2,$$
	over the open set of $k\times k$ invertible matrices $B$. Clearly $B=I_k$ is a local minimizer of $\varphi$ and therefore $\nabla \varphi(I_k)$ must be the zero matrix. A quick computation yields the expression 
	$\nabla\varphi(I_k)=L^{\top}L-R^{\top}R,$
	and therefore  $(L,R)$ is balanced, as claimed.
\end{proof}

\section{Convex relaxation and regularity of flat solutions}\label{sec:conv_relax_flat}
In this section, we begin investigating flat minimizers of the problem \eqref{eq: main minimization} with general linear measurement maps $\mathcal{A}$. It will be convenient to write the linear map $\mathcal{A}(X)$ in coordinates as
$$\mathcal{A}(X)=(\langle A_1,X\rangle,\langle A_2,X\rangle\ldots, \langle A_m,X\rangle),$$
where $A_i\in \RR^{d_1\times d_2}$ are some matrices. As always, $\mathcal{S}$ denotes the set of solutions to the equation $\mathcal{A}(LR^{\top})=b$. We will make use of the following two 	``rescaling'' matrices: 
\begin{equation}\label{eq: Didef}
	D_1 = \left(\frac{1}{\ncons \dmtwo} \sum_{i=1}^m A_i A_i^\top \right)^{\frac{1}{2}}\quad \text{and}\quad D_2 = \left(\frac{1}{\ncons \dmone} \sum_{i=1}^m A_i^\top  A_i\right)^{\frac{1}{2}}.
\end{equation}
The section presents two main results: Theorems~\ref{thm: equivalenceConvexNonconvexScaledTraceHessioan} and~\ref{thm: regularity}. The former presents a convex relaxation for verifying that a solution is flat, while the latter shows that flat solutions are nearly balanced and nearly norm-minimal, whenever the matrices $D_1$ and $D_2$ are well-conditioned.

\subsection{A convex relaxation for flat solutions}\label{sec: A_convex_relaxation_for_flat_solutions}

Flat solutions are by definition minimizers of the highly nonconvex problem $\min_{(L,R)\in \mathcal{S}}\str(D^2 f(L,R)).$ The main result of this section is to present an appealing convex relaxation of this problem. We begin with a convenient expression for the scaled trace $\str(D^2 f(L,R))$. Namely, 
recall that Lemma~\ref{lem:flattest} showed the equality $\str(D^2 f(L,R))=2\fronorm{L}^2+2\fronorm{R}^2$ in the simplified setting $\mathcal{A}=\mathcal{I}$. Lemma~\ref{lem: scaledtraceFronormSquareSum} provides an analogous statement for general maps $\mathcal{A}$ up to rescaling the factors by $D_1$ and $D_2$. 

\begin{lemma}[Scaled trace and the Frobenius norm]\label{lem: scaledtraceFronormSquareSum}
	The second-order derivative of the function $f$ at any $(L,R)\in \mathcal{S}$ is the quadratic form:
	\begin{equation}\label{eq:general}
		\begin{aligned}
			D^2 f(L,R)[U,V] & = 2\|\mathcal{A}(LV^\top +UR^\top)\|^2_2.
		\end{aligned}
	\end{equation}
	Moreover, the scaled trace can be written as
	\begin{equation}\label{eqn:scaled_trace_general}
		\str(D^2 f(L,R)) =2m(\fronorm{D_1L}^2 + \fronorm{D_2R}^2),
	\end{equation}
\end{lemma}
\begin{proof}
	An elementary computation yields for any $(L,R)$ the expression 
	$$D^2 f(L,R)[U,V]  = 4\langle \mathcal{A}(LR^{\top})-b, \mathcal{A}(LV^\top +UR^\top)\rangle +2\|\mathcal{A}(LV^\top +UR^\top)\|^2_2.$$
	Noting that for any $(L,R)\in \mathcal{S}$ the first term on the right is zero yields the claimed expression \eqref{eq:general}. Next, we verify \eqref{eqn:scaled_trace_general} by a direct calculation. To this end, the definition of the scaled trace \eqref{eqn:sctrace} yields the expression 
	
	\begin{equation}
		\str(\Hessian(L,R)) 
		=  \frac{2}{ \dmone} \sum_{i=1}^{\dmone}\sum_{j=1}^k \twonorm{\Amap(e_i e_j^\top R^\top )}^2+\frac{2}{ \dmtwo} \sum_{i=1}^{\dmtwo}\sum_{j=1}^k \twonorm{\Amap(L e_j e_i^\top )}^2 
		.
	\end{equation}
	Let us analyze the second term on the right. Letting $A_{l,i}$ denote the $i$'th column of $A_l$, we compute
	\begin{equation}\label{eqn:need_later}
		\begin{aligned}
			\sum_{i=1}^{\dmtwo}\sum_{j=1}^k  \twonorm{\Amap(L e_j e_i^\top )}^2  & = 
			\sum_{i=1}^{\dmtwo}\sum_{j=1}^k \sum_{l=1}^m 
			\inprod{A_l}{L e_j e_i^\top}^2\\
			&=  \sum_{i=1}^{\dmtwo}\sum_{j=1}^k \sum_{l=1}^m  
			\langle A_{l,i}, L_{j}\rangle^2\\
			&=\sum_{i=1}^{\dmtwo}\sum_{j=1}^k \sum_{l=1}^m  
			\langle A_{l,i} A_{l,i}^{\top}, L_{j}L_{j}^{\top}\rangle\\
			&= \sum_{l=1}^m 
			\langle \sum_{i=1}^{\dmtwo}  A_{l,i} A_{l,i}^{\top}, \sum_{j=1}^k L_{j}L_{j}^{\top}\rangle=\sum_{l=1}^m \langle A_lA_l^{\top}, LL^{\top}\rangle=m d_2\fronorm{D_1L}^2.
		\end{aligned}
	\end{equation}
	A similar argument shows $\fronorm{D_2R}^2 = \frac{1}{\ncons \dmone} \sum_{i=1}^{\dmone}\sum_{j=1}^k \twonorm{\Amap(e_i e_j^\top R^\top )}^2$, completing the proof.
\end{proof}

In particular, Lemma~\ref{lem: scaledtraceFronormSquareSum} implies that flat solutions are exactly the minimizers of the problem 
\begin{equation}\label{eq: scaled trace Sum of Square}
	\min_{L,R} ~
	\frac{1}{2}\left(\fronorm{D_1L}^2 + \fronorm{D_2R}^2\right)\qquad \textrm{subject to}\qquad \Amap(LR^\top)=b.
\end{equation}
In turn, it follows directly from \eqref{eqn:srebro} that so long as $D_1$, $D_2$ are invertible, the problem \eqref{eq: scaled trace Sum of Square} is equivalent to minimizing the nuclear norm over rank constrained matrices:
\begin{equation}\label{eq: PerturbnucnormX}
	\min_{X\in \RR^{d_1\times d_2}:~\rank(X)\leq k} \nucnorm{X}\qquad\textrm{subject to}\qquad \Amap(D_1^{-1}XD_2^{-1})=b. 
\end{equation}
Therefore, a natural convex relaxation for finding the flattest solution drops the rank constraint: 
\begin{equation}\label{eq: PerturbnucnormX_con_relax}
	\min_{X\in \RR^{d_1\times d_2}} \nucnorm{X}\qquad\textrm{subject to}\qquad \Amap(D_1^{-1}XD_2^{-1})=b. 
\end{equation}
The following theorem summarizes these observations.

\begin{thm}[Convex relaxation]\label{thm: equivalenceConvexNonconvexScaledTraceHessioan}
	Suppose the matrices $D_1$ and $D_2$ are invertible. Then the problems \eqref{eq: scaled trace Sum of Square} are \eqref{eq: PerturbnucnormX} are equivalent in the following sense. Let $l = \min(k,\dm_{\min})$.
	\begin{enumerate}
		\item the optimal values of \eqref{eq: scaled trace Sum of Square} and \eqref{eq: PerturbnucnormX} are equal,
		\item if $L,R$ solves \eqref{eq: scaled trace Sum of Square}, then $X=D_1LR^\top D_2$ is a minimizer of \eqref{eq: PerturbnucnormX}. 
		\item if a solution  $X$ of  \eqref{eq: PerturbnucnormX} has a singular value decomposition $X= U\Sigma V^\top$ for some diagonal matrix $\Sigma \in \real^{l\times l}$ with nonnegative entries, then the matrices $L = D_1^{-1}U\sqrt{\Sigma}$ and $R= D_2^{-1} V\sqrt{\Sigma}$ are minimizers of \eqref{eq: scaled trace Sum of Square} when $l\geq k$, and the matrices $L = [D_1^{-1}U\sqrt{\Sigma},0_{\dm_1 ,(k-l)}]$ and $R= [D_2^{-1} V\sqrt{\Sigma}, 0_{\dm_2 ,(k-l)}]$ are minimizers of \eqref{eq: scaled trace Sum of Square} when $l<k$.
	\end{enumerate}
	Moreover, if $X=D_1M_{\natural}D_2$ is the unique minimizer of  the problem \eqref{eq: PerturbnucnormX_con_relax}, then any flat solution $(L,R)$ satisfies $LR^\top = M_{\natural}$. 
\end{thm}
\begin{proof}
	The three claims follow directly from making a variable substitution $L'=D_1 L$ and $R'=D_2R$ and using  \eqref{eqn:srebro}.  The ``moreover'' part follows from \eqref{eq: PerturbnucnormX_con_relax} being a convex relaxation of \eqref{eq: PerturbnucnormX}.
\end{proof}

Section~\ref{sec:RIP} will verify that the convex relaxation \eqref{eq: PerturbnucnormX_con_relax} indeed recovers $M_{\natural}$ under restricted isometry properties on the measurement map $\mathcal{A}$, and therefore flat solutions exactly recover $M_{\natural}$.

\subsection{Regularity of flat solutions}
In this section, we show that the condition numbers of the rescaling matrices $D_1$ and $D_2$ determine balancedness and norm minimality of  flat solutions. The main result is the following theorem.

\begin{thm}[Regularity of flat solutions]\label{thm: regularity}
	Suppose that there exist constants $\alpha_1,\alpha_2> 0$ satisfying $\alpha_1 I \preceq D_i \preceq \alpha_2 I$ for each $i\in \{1,2\}$.  Define the constant $\kappa := \frac{\alpha_2}{\alpha_1}$.
	Then any flat solution $(L_f,R_f)$ of \eqref{eq: main minimization}
	satisfies the following properties.
	\begin{enumerate}
		\item 	{\bf Norm-minimal}: 
		the pair $(L_f,R_f)$ is approximately norm-minimal: 
		\begin{equation}
			\begin{aligned}\label{eq: inequalityLmRmLfRf}
				\fronorm{L_f }^2 + \fronorm{R_f}^2  
				\leq 
				\kappa^2\cdot\left(\min_{\Amap(LR^\top) = b} \fronorm{L}^2 + \fronorm{R}^2\right).
			\end{aligned}
		\end{equation}
		\item {\bf Balanced}: The pair $(L_f,R_f)$ is approximately balanced: 
		\begin{equation}\label{eq: balancednessdeterministic}
			\nucnorm{L^\top_f L_f - R_f^\top R_f} \leq
			2\left(\kappa^2- 1\right) \nucnorm{\truM}.
		\end{equation} 
	\end{enumerate}
\end{thm}

The proof of Theorem~\ref{thm: regularity} relies on the following simple linear algebraic lemma. 

\begin{lem}\label{lem: balancedness}
	Consider two symmetric matrices $Q_1\in \RR^{d_1\times d_1}$ and $Q_2\in \RR^{d_2\times d_2}$. Suppose that there exist constants $\alpha_1, \alpha_2>0$ satisfying
	$\alpha_1 I \preceq Q_i \preceq \alpha_2 I$ for each $i\in \{1,2\}$. Define the constant $\kappa = \frac{\alpha_2}{\alpha_1}$.
	Then given any matrix $X\in \RR^{d_1\times d_2}$, any minimizer $(L,R)$ of the problem
	\begin{equation}
		\min_{\tilde{L},\tilde{R}: \;Q_1\tilde{L}\tilde{R}^{\top} Q_2 = X} \frac{1}{2}\left( \fronorm{Q_1\tilde{L}}^2 + \fronorm{Q_2\tilde{R}}^2\right), 
	\end{equation}
	satisfies the inequality: 
	\begin{equation}\label{eq: LRbalanceLsquareplucRsquare}
		\nucnorm{L^\top L - R^\top R} \leq \left(1- \kappa^{-2}\right)\left(\fronorm{L}^2+\fronorm{R}^2\right) .
	\end{equation}
\end{lem}
\begin{proof}
	Lemma \ref{lem: balancedMinimalNorm} implies that the pair $(Q_1L,Q_2R)$ is balanced, meaning  
	$L^\top Q_1^2L  = R^\top Q_2^2 R$.
	Hence, we may decompose $L^\top L- R^\top R$ in the following way: 
	\begin{equation}
		L^\top L- R^\top R = 	\left(L^\top L- \frac{L^\top Q_1^2L}{\alpha_2^2} \right) + \left(\frac{R^\top Q_2^2 R}{\alpha_2^2} - R^\top R\right).
	\end{equation}
	We bound the first term on the right as follows,
	\begin{equation}\label{eq: boundingLdifference}
		\begin{aligned}
			\nucnorm{L^\top L- \frac{L^\top Q_1^2L}{\alpha_2^2}} = \nucnorm{L^\top \left(I-\frac{1}{\alpha_2^2}Q_1^2\right)L} 
			&\overset{(a)}{\leq}\fronorm{L^\top \left(I-\frac{1}{\alpha_2^2}Q_1^2\right)}\fronorm{L}\\
			&\overset{(b)}{\leq} \opnorm{I-\frac{1}{\alpha_2^2}Q_1^2} \fronorm{L}^2 \\ 
			&\leq\left(1- \kappa^{-2}\right)\fronorm{L}^2.
		\end{aligned}
	\end{equation}
	Here, $(a)$ and $(b)$ follow, respectively, from the basic inequalities:  
	$\nucnorm{FG}\leq \fronorm{F}\fronorm{G}$ and $\|FG\|_{F}\leq \opnorm{F}\fronorm{G}$, which hold for all matrices $F$ and $G$ with compatible dimensions.  A similar 
	argument yields the inequality 
	\begin{equation}\label{eq: boundingRdifference}
		\begin{aligned}
			\nucnorm{R^\top R- \frac{R^\top Q_2^2R}{\alpha_2^2}}
			\leq \left(1- \kappa^{-2}\right)\fronorm{R}^2.
		\end{aligned}
	\end{equation}
	The claimed estimate \eqref{eq: LRbalanceLsquareplucRsquare} follows immediately.
\end{proof}

We are now ready to prove Theorem~\ref{thm: regularity}.

\begin{proof}[Proof of Theorem~\ref{thm: regularity}]
	We first prove inequality \eqref{eq: inequalityLmRmLfRf}. To this end, for any $(L,R)\in \mathcal{S}$, we successively estimate:
	\begin{equation}
		\begin{aligned}
			{\alpha_1^2}\left( \fronorm{L_f }^2 + \fronorm{R_f}^2 \right) 
			& \leq
	\fronorm{D_1 L_f }^2 + \fronorm{D_2R_f}^ 2 \\
			& \leq	
\fronorm{D_1 L }^2 + \fronorm{D_2R}^2 \\
			& \leq 	
		{\alpha_2^2}\left( \fronorm{L}^2 + \fronorm{R}^2 \right),
		\end{aligned}
	\end{equation}
	where the second inequality follows from the characterization \eqref{eq: scaled trace Sum of Square} of flat solutions. Taking the infimum over pairs $(L,R)\in\mathcal{S}$ completes the proof of \eqref{eq: inequalityLmRmLfRf}.
	
	We next verify \eqref{eq: balancednessdeterministic}. 
	To this end, define the matrix $X:=D_1L_fR_f^\top D_2$. Then clearly $(L_f,R_f)$ is a minimizer of the problem 
	\begin{equation}
		\min_{\tilde{L},\tilde{R}:~ \;D_1\tilde{L}\tilde{R}^{\top} D_2 = X} \frac{1}{2}\left( \fronorm{D_1\tilde{L}}^2 + \fronorm{D_2\tilde{R}}^2\right). 
	\end{equation}
	Lemma \ref{lem: balancedness} therefore guarantees the estimate
	$$
	\nucnorm{L^\top_f L_f - R_f^\top R_f} \leq \left(1- \kappa^{-2}\right)\left(\fronorm{L_f}^2+\fronorm{R_f}^2\right).
	$$
	The already established estimate \eqref{eq: inequalityLmRmLfRf} ensures
	$$\fronorm{L_f}^2+\fronorm{R_f}^2\leq \kappa^2 \left(\fronorm{L}^2 + \fronorm{R}^2\right)\qquad\forall (L,R)\in \mathcal{S}.$$
	In particular, 
	minimizing the right hand-side over $L,R$ satisfying $M_{\natural}=LR^{\top}$ yields an upper bound of $2\kappa^2\nucnorm{M_{\natural}}$.
	The proof is complete.
\end{proof}

\section{Flat minima under RIP conditions: matrix and bilinear sensing}
\label{sec:RIP}

The previous section motivates a two-part strategy for showing that flat minima exactly recover the ground truth (Theorem~\ref{thm: equivalenceConvexNonconvexScaledTraceHessioan}) and are automatically nearly balanced and nearly norm-minimal (Theorem~\ref{thm: regularity}). The first step is to argue that the convex relaxation \eqref{eq: PerturbnucnormX_con_relax} admits $D_1M_{\natural}D_2$ as its unique minimizer. The second step is to argue that the condition numbers of the matrices $D_1$ and $D_2$ are close to one. In this section, we follow this recipe for problems satisfying $\ell_2/\ell_2$ and $\ell_1/\ell_2$ restricted isometry properties (defined below). The main two examples will be the following random ensembles.

\begin{definition}[Matrix and bilinear sensing]\label{defn:mat_sense_bilin}
	{\rm
		We introduce the following definitions.
		\begin{enumerate}
			\item  We  say that $\mathcal{A}$ is a {\em Gaussian  ensemble} if the entries of $A_i$ are i.i.d standard normal random variables $N(0,1)$.
			\item We  say that $\mathcal{A}$ is a {\em Gaussian bilinear  ensemble} if the matrices $A_i$ take the form $A_i = a_ib_i^\top$ where the entries of $a_i$ and $b_i$ are i.i.d. standard normal random variables $N(0,1)$
	\end{enumerate}}
\end{definition}

The main results of the section is the following theorem, stated here informally.

\begin{thm}[Matrix and bilinear sensing (Informal)]\label{thm: exact recovery}
	Suppose we are in one of the settings:
	\begin{enumerate}
		\item  $\mathcal{A}$ is a Gaussian  ensemble  and $m\gtrsim \trur d_{\max}$, 
		\item $\mathcal{A}$ is a Gaussian bilinear  ensemble, $m\gtrsim \trur d_{\max}$, and $d_{\min}\gtrsim \log \ncons$.
	\end{enumerate}
	Then with high probability, any flat solution $(L_f,R_f)$ of \eqref{eq: main minimization} satisfies 
	$
	L_fR_f^\top = \truM,
	$ and is nearly norm-minimal and nearly balanced:
	\begin{align*}
		\fronorm{L_f }^2 + \fronorm{R_f}^2  
		&
		\leq (1+\delta)\left(\min_{\Amap(LR^\top) = b} \fronorm{L}^2 + \fronorm{R}^2\right)\\
		\nucnorm{L^\top_f L_f - R_f^\top R_f} &\lesssim
		\delta\nucnorm{\truM}.
	\end{align*}
\end{thm}

We begin by formally defining the restricted isometry property of a measurement map $\mathcal{A}(\cdot)$.

\begin{definition}[Restricted isometry property]\label{def: RIP}
	{\rm	A linear map $\Amap\colon\real^{\dmone \times \dmtwo}\rightarrow \real^{\ncons}$ satisfies an {\em $\ell_p/\ell_2$ restricted isometry property (RIP) with parameters $(r,\delta_1,\delta_2)$} if the estimate 
		\begin{equation}\label{eq: ripcondition}
			\delta_1\fronorm{X}\leq \frac{\|\Amap(X)\|_p}{m^{1/p}}\leq \delta_2\fronorm{X}, 
		\end{equation}
		holds for all matrices $X\in \real^{\dmone \times \dmtwo}$ with rank at most $r$.}
\end{definition}
In this paper we will be primarily interested in $\ell_{2}/\ell_{2}$ and $\ell_{1}/\ell_{2}$ restricted isometry properties. In particular, the two random measurement models in Theorem~\ref{thm: exact recovery} satisfy these two properties. The following two lemmas are from \cite[Theorem 2.3]{candes2011tight}, \cite[Theorem 4.2]{recht2010guaranteed}, and \cite[Theorem 2.2]{cai2015rop}.

\begin{lemma}[$\ell_2/\ell_2$ RIP in matrix sensing]
	\label{lem:mat_sense_22}
	Let  $\mathcal{A}$ be a Gaussian  ensemble. Then  
	for any $\delta\in (0,1)$, there exist constants $c,C>0$ depending only on $\delta$ such that as long as $m\geq cr(d_1+d_2)$, with probability at least $1-\exp(-C m)$, the linear map $\mathcal{A}$ satisfies $\ell_2/\ell_2$ RIP with parameters $(2r,1-\delta,1+\delta)$.
\end{lemma}

\begin{lemma}[$\ell_1/\ell_2$ RIP in bilinear sensing]\label{lem:RIP for bilinear sensing}
	Let  $\mathcal{A}$ be a Gaussian bilinear ensemble. For any positive integer $k\geq 2$, there exist constants $c,C>0$ depending only on $k$ and numerical constants $\delta_1,\delta_2>0$ and  such that in the regime $m\geq cr(d_1+d_2)$, with probability at least $1-\exp(-C m)$, the measurement map $\mathcal{A}$ satisfies $\ell_1/\ell_2$ RIP with parameters $(kr, \delta_1, \delta_2)$.
\end{lemma}

Our goal is to show that under RIP conditions, with reasonable parameters, flat solutions exactly recover the ground truth $M_{\natural}$.
We will need the following lemma, whose proof is immediate from definitions.

\begin{lemma}\label{lem:reparam}
	Let $\mathcal{A}(\cdot)$ be a linear map satisfying an $\ell_p/\ell_2$ RIP with parameters $(r,\delta_1,\delta_2)$. Let $Q_1,Q_2$ be two positive definite matrices satisfying $ \alpha _1 I\preceq Q_i\preceq  \alpha _2I$ for all $i\in\{1,2\}$ for some $\alpha_1,\alpha_2>0$. Then the linear map  $\mathcal{B}(\cdot) :\,= \mathcal{A}(Q_1^{-1}\cdot Q_2^{-1})$ satisfies an $\ell_p/\ell_2$ RIP with parameters $(r,\alpha_2^{-2}\delta_1,\alpha_1^{-2}\delta_2)$.
\end{lemma}

The following lemma will be our main technical tool; it establishes that if $\mathcal{A}(\cdot)$ satisfies RIP, then so does the perturbed map  $\mathcal{B}(\cdot) =\mathcal{A}(Q_1^{-1}\cdot Q_2^{-1})$, provided that the condition numbers of the positive definite matrices $Q_1$ and $Q_2$  are sufficiently close to one.

\begin{lemma}\label{lem: exactrecoverConvex}
	Consider two positive definite matrices $Q_1,Q_2$ and constants $\alpha_1,\alpha_2>0$ satisfying $\alpha _1 I\preceq Q_i\preceq  \alpha _2I$ for each  $i\in\{1,2\}$. Define $\kappa =\alpha_2/\alpha_1$ and let $\mathcal{A}(\cdot)$ be a linear map satisfying one of the following conditions.
	\begin{enumerate}
		\item The map $\Amap$ satisfies $\ell_2/\ell_2$ RIP with parameters $(5\trur, \delta_1,\delta_2)$, where  $\delta_1\geq \frac{9}{10}$ and $\delta_2\kappa^2\leq \frac{11}{10}$.
		\item The map $\Amap$ satisfies $\ell_1/\ell_2$ RIP with parameters $(l\trur, \delta_1,\delta_2)$, where $\frac{\delta_2\kappa^2}{\delta_1}< \sqrt{l}$.
	\end{enumerate}
	Then $Q_1\truM Q_2$ is the unique solution of the following convex program 
	\begin{equation}\label{eq: convexRelaxationPerturb}
		\min_{X\in \RR^{d_1\times d_2}} \nucnorm{X}\qquad \textrm{subject to}\qquad\Amap(Q_1^{-1}X Q_2^{-1})=b. 
	\end{equation}
\end{lemma}
\begin{proof}
	Define the map $\mathcal{B}(Z)=\alpha_2^2\Amap(Q_1^{-1}ZQ_2^{-1})$. Then Lemma~\ref{lem:reparam} implies that $\mathcal{B}$ in the first case satisfies $\ell_2/\ell_2$ RIP with parameters $(5r_{\natural},\delta_1,\delta_2\kappa^{2})$ and in the second case satisfies $\ell_1/\ell_2$ RIP with parameters $(l r_{\natural},\delta_1,\delta_2\kappa^{2})$. An application of \cite[Theorem 3.3]{recht2010guaranteed} in the first case and 
	\cite[Theorem 2.1]{cai2015rop} in the second guarantees that $Q_1\truM Q_2$ is the unique solution of \eqref{eq: convexRelaxationPerturb}, as claimed.
\end{proof}

It remains to estimate the condition number $\kappa$ of the matrices $D_1$ and $D_2$ defined in \eqref{eq: Didef} under RIP (or statistical assumptions). The following lemma shows that $D_1$ and $D_2$ are automatically well conditioned under $\ell_2/\ell_2$ RIP.

\begin{lemma}[Conditioning of $D_i$ under $\ell_2/\ell_2$ RIP]\label{lem:cond_num}
	Suppose that the linear map $\mathcal{A}$ satisfies $\ell_2/\ell_2$ RIP with parameters $(1,\delta_1,\delta_2)$. Then the relation $\delta_1 I_{d_i}\preceq D_i \preceq \delta_2 I_{d_i}$ holds for all $i\in \{1,2\}$
\end{lemma}
\begin{proof}
	In the proof of Lemma~\ref{lem: scaledtraceFronormSquareSum} (equation \eqref{eqn:need_later}), we actually showed the expression: 
	\[
	\sum_{i=1}^{\dmtwo}\sum_{j=1}^k  \twonorm{\Amap(L e_j e_i^\top )}^2 = m d_2\fronorm{D_1L}^2\qquad \forall L\in \real^{\dmone \times k}.
	\]
	Now for any vector $v\in \real^{\dmone}$, we can take the matrix $L=[v,0_{d_1,k-1}]$ in the above equation. Appealing to the $\ell_2/\ell_2$ RIP condition on $\Amap$, we deduce
	\[
	\delta_1^2 \twonorm{v}^2\leq \twonorm{D_1v}^2 \leq \delta_2^2  \twonorm{v}^2.
	\]
	A similar argument shows that $D_2$ satisfies the analogous inequality with $v\in \real^{\dmtwo}$.
\end{proof}

The $\ell_1/\ell_2$ RIP property does not in general imply a good bound on the condition numbers of $D_1$ and $D_2$. Instead we will directly show that under the Gaussian design for bilinear sensing, the matrices $D_1$ and $D_2$ are well-conditioned. This is the content of the following lemma.

\begin{lemma}[Conditioning of $D_i$ for bilinear sensing]\label{lem_condbilin}
	Let $\mathcal{A}$ be a Gaussian bilinear ensemble. Then there exist constant $c_1,c_2,c_3,c_4>0$ such that for any $\delta\in(0,1)$ as long as we are in the regime $m\geq \frac{c_3d_{\max}}{\delta^2}$ and $\log(m)\leq c_4\delta^2d_{\min}$, 
	the estimate holds:
	$$\mathbb{P}\left\{(1-\delta)I_{d_i}\preceq D_i\preceq(1+\delta)I_{d_i}~~\forall i\in\{1,2\}\right\}\geq 1-c_2e^{-c_1d_{\min}\delta^2}.$$
\end{lemma}
\begin{proof}
	First observe $A_iA_i^{\top}=\|b_i\|^2_2a_ia_i^{\top}$ for each index $i$. Bernstein's inequality \cite[Theorem 2.8.3]{vershynin2018high} implies 
	$$\mathbb{P}\left\{\left|\frac{1}{d_2}\|b_i\|^2_2-1\right|\geq \delta\right\}\leq 2\exp(-c_1d_2\delta^2)\qquad \forall \delta>0,~\forall i\in [m].$$
	Taking a union bound, we can therefore be sure that with probability at least $1-m\exp(-c_1d_2\delta^2)$ the estimate 
	$$1-\delta\leq \frac{\|b_i\|^2_2}{d_2}\leq 1+\delta$$
	holds simultaneously for all $i=1,\ldots, m$. In this event, we estimate
	$$(1-\delta)a_ia_i^\top  \preceq \frac{1}{d_2} A_iA_i^\top  \preceq (1+\delta)a_ia_i^\top$$
	Therefore, after summing for $i=1,\ldots, m$ we deduce
	$$(1-\delta) \frac{1}{\ncons } \sum_{i=1}^\ncons a_ia_i^\top \preceq D_1^2  \preceq (1+\delta_2)\frac{1}{\ncons } \sum_{i=1}^\ncons a_ia_i^\top.$$
	Concentration of covariance matrices \cite[Exercise 4.7.3]{vershynin2018high} in turn implies that the estimate 
	$$\left\| \frac{1}{m}A^{\top}A-I_{d_1}\right\|_{\rm op}\leq c_2\left(\sqrt{\frac{d_1+u}{m}}+\frac{d_1+u}{m}\right),$$
	holds with probability at least $1-2\exp(-u)$.
	Taking a union bound, we therefore deduce 
	$$(1-\delta)\left(1-c_2\left(\sqrt{\frac{d_1+u}{m}}+\frac{d_1+u}{m}\right)\right) I_{d_1}  \preceq D_1^2  \preceq (1+\delta_2)\left(1+c_2\left(\sqrt{\frac{d_1+u}{m}}+\frac{d_1+u}{m}\right)\right) I_{d_1}.$$
	holds with probability at least $1-m\exp(-c_1d_2\delta^2)-2\exp(-u)$. Setting $u=d_1$, we see that there is a constant $c_3$ such that as long as $m\geq c_3\frac{\max\{d_1,d_2\}}{\delta^2}$, we have 
	$$(1-\delta)^2 I_{d_1}\preceq D_1^2 \preceq (1+\delta)^2 I_{d_1}$$
	with probability at least $1-m\exp(-c_1d_2\delta^2)-2\exp(-d_1)$. The result follows.
\end{proof}

The following are the two main results of the section.
\begin{thm}[Exact recovery in matrix sensing]
	Suppose that $\mathcal{A}$ is a Gaussian ensemble. Then there exists a constant $c_0$ such that the following hold for any $\delta\in (0,c_0)$. There exist constants $c,C>0$ depending only on $\delta$ such that
	in the regime $m\geq c\trur (d_1+d_2)$, with probability at least $1-\exp(-C m)$, any flat solution $(L_f,R_f)$ of \eqref{eq: main minimization} satisfies $L_fR_f^\top = \truM$ and is automatically nearly norm-minimal and nearly balanced:
	\begin{align*}
		\fronorm{L_f }^2 + \fronorm{R_f}^2  
		&\leq 
		\left(\frac{1+\delta}{1-\delta}\right)^2\cdot\left(\min_{\Amap(LR^\top) = b} \fronorm{L}^2 + \fronorm{R}^2\right)\\
		\nucnorm{L^\top_f L_f - R_f^\top R_f} &\leq
		2\left(\left(\frac{1+\delta}{1-\delta}\right)^2- 1\right) \nucnorm{\truM}.
	\end{align*}

\end{thm}
\begin{proof}
	Lemma~\ref{lem:mat_sense_22} shows that for any $\delta\in (0,1)$, there exist constants $c_1,C_1>0$ depending only on $\delta$ such that as long as $m\geq c_1r(d_1+d_2)$, with probability at least $1-\exp(-C_1 m)$, the linear map $\mathcal{A}$ satisfies $\ell_2/\ell_2$ RIP with parameters $(r,1-\delta,1+\delta)$. In this event, Lemma~\ref{lem:cond_num} ensures that the condition number $\kappa$ of $D_1$ and $D_2$ is bounded by $\frac{1+\delta}{1-\delta}$. Set now $r=5r_{\natural}$ and choose any $\delta\leq 0.1$ satisfying $(1+\delta)\left(\frac{1+\delta}{1-\delta}\right)^2\leq \frac{11}{10}$. An application of Lemma~\ref{lem: exactrecoverConvex} and Theorem~\ref{thm: regularity} completes the proof.
\end{proof}

\begin{thm}[Exact recovery in bilinear sensing]
	Suppose that $\mathcal{A}$ is a Gaussian bilinear ensemble. Then for any $\delta\in (0,1)$ there exist numerical constants $c,C,c_1, c_2, c_3,c_4>0$ depending only on $\delta$ such that in the regime 
	$m\geq c \trur (d_1+d_2)$ and $\log(m)\leq c_4d_{\min}$, with probability at least $1-c_3\exp(-Cd_{\min})$ any flat solution $(L_f,R_f)$ of \eqref{eq: main minimization} satisfies 
	$L_fR_f^\top = \truM$ and is automatically nearly norm-minimal and nearly balanced:
	\begin{align*}
		\fronorm{L_f }^2 + \fronorm{R_f}^2  
		&\leq 
		\left(\frac{1+\delta}{1-\delta}\right)^2\cdot\left(\min_{\Amap(LR^\top) = b} \fronorm{L}^2 + \fronorm{R}^2\right)\\
		\nucnorm{L^\top_f L_f - R_f^\top R_f} &\leq
		2\left(\left(\frac{1+\delta}{1-\delta}\right)^2- 1\right) \nucnorm{\truM}.
	\end{align*}
	
\end{thm}
\begin{proof}
	For any integer $k\in \mathbb{N}$, Lemma~\ref{lem:RIP for bilinear sensing} ensures
	that there exist numerical constants $\delta_1,\delta_2>0$ and constants $c_0,C_0>0$ depending only on $l$ such that in the regime $m\geq c_0r(d_1+d_2)$, with probability at least $1-\exp(-C_0 m)$, the measurement map $\mathcal{A}$ satisfies $\ell_1/\ell_2$ RIP with parameters $(l\trur, \delta_1, \delta_2)$. Lemma~\ref{lem_condbilin} in turn ensures there exist numerical constant $c_5,c_6,c_7,c_8>0$ such that for any $\delta\in(0,1)$ as along as we are in the regime, $m\geq \frac{c_7d_{\max}}{\delta^2}$ and $\log(m)\leq c_8\delta^2 d_{\min}$, 
	the estimate holds:
	$$\mathbb{P}\left\{(1-\delta)I_{d_i}\preceq D_i\leq (1+\delta)I_{d_i}~~\forall i\in\{1,2\}\right\}\geq 1-c_5e^{-c_6d_{\min}\delta^2}.$$
	Therefore in this regime, we may upper bound the condition number $\kappa$ of $D_1$ and $D_2$ by $\frac{1+\delta}{1-\delta}$. In light of Lemma~\ref{lem: exactrecoverConvex}, in order to ensure exact recovery, it remains to simply choose a large enough $l$ such that the inequality $\frac{\delta_2}{\delta_1}\cdot (\frac{1+\delta}{1-\delta})^2\leq \sqrt{l}$ holds (recall $\delta_1,\delta_2$ are numerical constants).
	An application of Lemma~\ref{lem: exactrecoverConvex} and Theorem~\ref{thm: regularity} completes the proof.
\end{proof}

Appendix \ref{sec: noisyobservation} generalizes the material in this section to the noisy observation setting, wherein $b = \Amap(\truM) + e$ with $e\sim N(0,\sigma^2I)$ for some $\sigma^2>0$.

\section{Matrix completion and approximate recovery}\label{sec:mat_comp}
In this section, we focus on the matrix completion problem \cite{recht2011simpler,candes2009exact}. This is an instance of \eqref{eq: main minimization} where the linear measurement map $\mathcal{A}$ is generated as follows.
For each $i\in [d_1]$ and $j\in [d_2]$, let $\xi_{ij}$  be independent Bernoulli random variables with success probability $p$. The linear map $\Amap\colon\real^{\dmone\times \dmtwo}\to\real^{\dmone\times \dmtwo}$ is then defined by the relation 
\begin{equation}\label{eq: mcAmap}
	[\Amap (Z)]_{ij} = Z_{ij}\xi_{ij}\qquad \textrm \,\text{for any $(i,j)\in [\dmone]\times [\dmtwo]$}.
\end{equation}
The difficulty of recovering the matrix $\truM$ is typically measured by an incoherence parameter, which we now define. Given a singular value decomposition $\truM = \truU \truSig \truV^\top$ with $\truSig\in \real^{\trur \times \trur}$, the {\em incoherence parameter} is the smallest $\mu>0$ satisfying
\begin{equation}\label{eq: incoherenceMu}
	\twoinfnorm{\truU}\leq \sqrt{\frac{\mu \trur}{\dmone}}, \quad \text{and}\quad \twoinfnorm{\truV}\leq \sqrt{\frac{\mu \trur}{\dmtwo}}.
\end{equation}
Here $\twoinfnorm{A}$ denotes the maximal $\ell_2$-norm of the rows of the matrix $A$. The strategies outlined in the previous section do not directly apply for analyzing flat minima of the matrix completion problem because the linear map $\Amap(D_1^{-1}\cdot D_2^{-1})$ does not satisfy RIP type conditions. Instead we will settle for a weak recovery result.

\begin{thm}[Recovery error of flat solution]\label{thm: recoveryErrorMC}
	Suppose that $\mathcal{A}$ is a random sampling map described 
	in \eqref{eq: mcAmap}. Then there exist numerical constants $c,C>0$ such that the following is true. Given any $\gamma\in (0,1)$, provided we are in the regime 
	\begin{equation}\label{eqn:mat_comp_regime}
		p\geq C\max\left\{\frac{1}{\gamma}\sqrt{\frac{r_{\natural}\log(d_{\max})}{d_{\min}}}, \frac{\mu r_{\natural}\log(\mu r_{\natural})\log(d_{\max})}{d_{\min}}\right\},
	\end{equation}
	with probability at least $1-cd_{\min} ^{-5}$, any flat solution $(L_f,R_f)$ 	
	satisfies	\begin{equation} \label{eq: recoveryErrorMC}
		\nucnorm{L_f R_f^\top -\truM} \leq \gamma\nucnorm{\truM}.
	\end{equation}
\end{thm}

Hence according to Theorem~\ref{thm: recoveryErrorMC}, in order to conclude that flat solutions achieve a constant relative error, we must be in the regime $p\gtrsim \sqrt{\frac{\trur\log d_{\max}}{d_{\min}}}$. This is a stronger 
requirement than is needed for exact recovery of the ground truth matrix \cite{chen2015incoherence}, which is  $p\gtrsim \mu \trur \log(\mu \trur )\frac{\log (\dm_{\max})}{\dm_{\min}}$. Our numerical experiments, however, suggest that flat solutions exactly recover the ground truth.

As the first step towards proving Theorem~\ref{thm: recoveryErrorMC}, we estimate the condition numbers of $D_1$, $D_2$.

\begin{lemma}[Condition number]\label{lem:upper_bound_cond_num_matr_comp}
	For any given $\delta\in (0,1)$ and $c\geq 1$, as long as $p\geq \frac{(1+c)\log \left(d_{\max}\right)}{2\delta^2d_{\min}}$, with probability at least $1-4d^{-c}_{{\min}}$, the estimate
	\begin{equation}\label{eqn:get_conditionsingmat} 
		\sqrt{\frac{p}{\dmone \dmtwo} (1-\delta)} I_{d_i}\preceq D_i \preceq \sqrt{\frac{p}{\dmone \dmtwo} (1+\delta)} I_{d_i}\qquad \textrm{ holds for }i=1,2.
	\end{equation}
\end{lemma}
\begin{proof}
	Let $m=d_1d_2$ and set the sensing matrices $A_{ij}=\xi_{ij} e_ie_j^\top$ for all pairs $i\in [d_1]$ and $j\in [d_2]$. Therefore the equality $A_{ij}A_{ij}^{\top}=\xi_{ij}e_ie_i^{\top}$ holds, and we can write
	\begin{equation}
		\begin{aligned} 
			D_1^2  & =\frac{1}{\dmone \dmtwo^2} \sum_{i\in [d_1],\, j\in[d_2]}^m \xi_{ij} e_ie_i^\top = \frac{1}{\dmone \dmtwo^2} \diag\left(\sum_{j=1}^{\dmtwo}\xi_{1j},\dots, \sum_{j=1}^{\dmtwo}\xi_{\dmone j}\right).
		\end{aligned} 
	\end{equation}
	Bernstein's inequality \cite[Theorem 2.8.1]{vershynin2018high} implies
	for each index $i\in [d_2]$ the estimate
	$$\mathbb{P}\left\{\left|\frac{1}{p\dmtwo}\sum_{j=1}^{\dmtwo} \xi_{ij}-1\right|\geq \delta \right\}\leq 2\exp\left(-2pd_2 \delta^2\right)\qquad \forall \delta> 0.$$
	Taking the union bound over $i\in [d_1]$ we deduce that the condition 
	$$ \frac{p}{\dmone \dmtwo} (1-\delta) I\preceq D_1^2 \preceq \frac{p}{\dmone \dmtwo} (1+\delta) I$$
	fails with probability at most
	\begin{align*}
		d_1\exp\left(-2pd_2 \delta^2\right)&\leq \exp\left(-2pd_2 \delta^2+\log(d_1)\right)\leq  \exp\left(\frac{- pd_2 \delta^2}{4}\right)\leq d_2^{-c}.
	\end{align*}
	Using the same argument for $D_2$ and taking a union bound completes the proof. 
\end{proof}

Next, we will show that flat solutions are almost optimal for the standard convex relaxation of the matrix completion problem:
\begin{equation}\label{eq: convexRelaxationMC}
	\min_{\Amap(X) =\Amap(\truM)} \nucnorm{X}.
\end{equation}

\begin{lemma}[Flat minima  and nuclear norm minimization]\label{lemm:flatmat_comp}
	Suppose that $M_{\natural}$ is a solution of the problem \eqref{eq: convexRelaxationMC} and suppose that the condition numbers of $D_1$ and $D_2$ are upper bounded by some constant $\kappa>0$. Then any flat solution $(L_f,R_f)$ of \eqref{eq: main minimization} is nearly optimal for the convex relaxation \eqref{eq: convexRelaxationMC} in the sense that:
	\begin{equation}\label{eqn:subopt_gap}
		\nucnorm{L_fR_f^\top}-\nucnorm{M_{\natural}}\leq (\kappa^2-1)\nucnorm{M_{\natural}}.
	\end{equation}
\end{lemma}
\begin{proof}
	We successively estimate
	$$\nucnorm{L_fR_f^\top}\leq \frac{1}{2}\left(\fronorm{L_f}^2 + \fronorm{R_f}^2\right)\leq \kappa^2\left(\min_{\Amap(LR^\top) = \Amap(M_{\natural})} \frac{1}{2}\fronorm{L}^2 + \frac{1}{2}\fronorm{R}^2\right)=\kappa^2\nucnorm{M_{\natural}},$$
	where the first and last inequalities follow from \eqref{eqn:srebro} and the second inequality follows from Theorem \ref{thm: regularity}.
	We therefore deduce $\nucnorm{L_fR_f^\top}-\nucnorm{M_{\natural}}\leq (\kappa^2-1)\nucnorm{M_{\natural}}$, as claimed.
\end{proof}

It remains to translate the suboptimality gap \eqref{eqn:subopt_gap} into an estimate on the distance $\nucnorm{L_f R_f^{\top}-M_{\natural}}$. This is the content of the following lemma.

\begin{lemma}[Sharp growth in matrix completion]\label{lem: errorboundMC} 
	Suppose the linear map $\Amap$ is generated according to the matrix completion model. Then there exist constants $c,c_1,C>0$ such that in the regime $p\geq \frac{C\mu \trur \log (\mu\trur)\log (d_{\max})}{d_{\min}}$, with probability at least $1-c_1d_{\min}^{-5}$, any feasible matrix $X$ of the problem \eqref{eq: convexRelaxationMC} satisfies the  inequality:
	\begin{equation} \label{eq: recoveryEBMC}
		\nucnorm{X -\truM} \leq 8\left(1+\sqrt{\tfrac{6\trur}{p}}\right) (\nucnorm{X}-\nucnorm{\truM}).
	\end{equation}
\end{lemma}
\begin{proof}
	Let $M_{\natural}=U_{\natural}\Sigma_{\natural}V^{\top}_{\natural}$ be a singular value decomposition of $M_{\natural}$, where $\Sigma_{\natural}\in \RR^{r_{\natural}\times r_{\natural}}$ is a diagonal matrix of singular values.
	Define the matrix $R:= X - \truM$ and define the linear map $P_{\mathcal{T}}$ by
	\begin{equation}
		P_{\mathcal{T}} (Z) :\,= \truU\truU^\top Z+
		Z\truV \truV^\top - \truU \truU^\top Z \truV \truV^\top,
	\end{equation}
	Set $P_{\mathcal{T}^\perp}(Z):=Z- P_{\mathcal{T}}(Z)$.  Observe that we may bound $\nucnorm{R}$ as follows: 
	\begin{equation}\label{eq: nucRbound}
		\nucnorm{R} =\nucnorm{P_{\mathcal{T}^\perp}(R) + P_{\mathcal{T}}(R)}\leq
		\nucnorm{P_{\mathcal{T}^\perp}(R)} +\nucnorm{P_{\mathcal{T}}(R)}
		\overset{(a)}{\leq} \nucnorm{P_{\mathcal{T}^{\perp}}(R)}+ \sqrt{3\trur} \fronorm{P_{\mathcal{T}}(R)} ,
	\end{equation}
	where the step $(a)$ is due to the fact that $P_{\mathcal{T}}(R)$ has rank no more than $3\trur$. 
	We now bound $\nucnorm{P_{\mathcal{T}^{\perp}}(R)}$ and $\fronorm{P_{\mathcal{T}}(R)}$ separately. 
	As verified in \cite[Section 6]{ding2020leave}, the premise in \cite[Proposition 2]{chen2015incoherence}
	is satisfied with  probability at least $1-c_3\dmone^{-5} - c_3\dmtwo^{-5}$ for some universal $c_3>0$ under 
	the condition $p\geq \frac{C\mu \trur \log (\mu\trur)\log (d_{\max})}{d_{\min}}$. Hence,
	the result \cite[Proposition 2 and its proof]{chen2015incoherence}\footnote{Specifically, the first displayed equation above \cite[Lemma 5]{chen2015incoherence}} shows that  with probability at least $1-c_3\dmone^{-5} - c_3\dmtwo^{-5}$, there holds the inequality
	\begin{equation}\label{eq: objProjectionPerpBoound}
		\nucnorm{P_{\mathcal{T}^\perp}(R)}\leq 8(\nucnorm{X}- \nucnorm{\truM}). 
	\end{equation}
	Moreover, the premise in \cite[Lemma 5]{chen2015incoherence} is satisfied with probability at least $1-c_4\dmone^{-5} - c_4\dmtwo^{-5}$ for some universal constant $c_4>0$ as verified in \cite[Lemma 4.1]{candes2009exact} 
	or in \cite[Lemma 11]{chen2013low}.
	Hence, \cite[Lemma 5 and its proof]{chen2015incoherence}\footnote{In the displayed equation in the statement of the lemma, one can simply replace $n^5$ by $\frac{1}{\sqrt{p}}$ and set $Z=R$.} shows that with probability at least $1-c_4\dmone^{-5} - c_4\dmtwo^{-5}$, the inequality
	\begin{equation}\label{eq: projectionAndProjectionPerpBoound}
		\fronorm{ P_{\mathcal{T}}(R)}\leq  \frac{\sqrt{2}}{\sqrt{p}}\nucnorm{P_{\mathcal{T}^\perp}(R)}.
	\end{equation}
	holds.
	Combining \eqref{eq: nucRbound},\eqref{eq: objProjectionPerpBoound}, and \eqref{eq: projectionAndProjectionPerpBoound}, yields the desired inequality \eqref{eq: recoveryEBMC}. 
\end{proof}

Putting all the lemmas together, we can now prove Theorem~\ref{thm: recoveryErrorMC}.
\begin{proof}[Proof of Theorem~\ref{thm: recoveryErrorMC}]
	Lemma~\ref{lem: errorboundMC} ensures that in the regime $p\geq C \frac{\mu \trur \log (\mu\trur)\log (d_{\max})}{d_{\min}}$, with probability at least $1-c_1d_{\min}^{-5}$, the estimate 
	$$\nucnorm{X -\truM} \leq 8\left(1+\sqrt{\tfrac{6\trur}{p}}\right) (\nucnorm{X}-\nucnorm{\truM}),$$
	holds for all $X$ satisfying $\mathcal{A}(X)=\mathcal{A}(M_{\natural})$. In this event, $M_{\natural}$ is clearly a minimizer of \eqref{eq: convexRelaxationMC}. Lemma~\ref{lemm:flatmat_comp} therefore ensures that the matrix $X:=L_fR_f^\top$ satisfies 
	$$\nucnorm{X}-\nucnorm{M_{\natural}}\leq (\kappa^2-1)\nucnorm{M_{\natural}},$$
	where $\kappa$ is an upper bound on the condition numbers of $D_1$ and $D_2$.  Lemma~\ref{lem:upper_bound_cond_num_matr_comp} in turn ensures that for any $\delta\in (0,1)$, in the regime 
	$p\geq \frac{3\log \left(d_{\max}\right)}{\delta^2d_{\min}}$, with probability at least $1-4d_{\min}^{-5}$, the upper bound $\kappa\leq \sqrt{\frac{1+\delta}{1-\delta}}$ is valid. Algebraic manipulations therefore yield, within these events, the estimate:
	\begin{equation} \label{eq: recoveryErrorMC2}
		\nucnorm{L_f R_f^\top -\truM} \leq C\delta\sqrt{\frac{\trur}{p }}\nucnorm{\truM},
	\end{equation}
	for a some numerical constant $C>0$.
	To summarize, there exist numerical constants $c_1,c_2,C>0$ such that the following is true. Given any $\delta\in (0,1)$, provided we are in the regime 
	\begin{equation}\label{eqn:need_dis}	
		p\geq c_1 \max\{\mu r_{\natural}\log(\mu r_{\natural}),\delta^{-2}\}\cdot\frac{\log(d_{\max})}{d_{\min}},	
	\end{equation}
	with probability at least $1-c_2d_{\min} ^{-5}$, any flat solution $(L_f,R_f)$ 	
	satisfies \eqref{eq: recoveryErrorMC2}. Let us now try to set
	$$\delta^{-2}=\max\left\{1,\mu r_{\natural}\log(\mu r_{\natural}),\frac{C^2 r_{\natural}}{\gamma^2p}\right\}.$$
	This choice is consistent with the requirement \eqref{eqn:need_dis}	 as long as \eqref{eqn:mat_comp_regime} holds.
	With this choice of $\delta$, the estimate \eqref{eq: recoveryErrorMC2} becomes
	$\nucnorm{L_f R_f^\top -\truM} \leq \gamma\nucnorm{\truM}$, as claimed.
\end{proof}

\section{Robust principal component analysis (PCA)}\label{sec: extension}
In this section, we focus on problem of principal component analysis (PCA) with outliers, also known as ``robust PCA", following  the approach in \cite{candes2011robust,chandrasekaran2011rank}. Though this problem is not of the form \eqref{eq: main minimization}, we will see that flat solutions (appropriately defined) exactly recover the ground truth under reasonable assumptions.
The robust PCA problem asks to find a matrix $\truM \in \real^{\dmone\times \dmtwo}$ that has been corrupted by sparse noise $\truS$. More precisely, we observe a matrix  $Y\in \real^{\dmone\times \dmtwo}$ of the form 
\begin{equation*}
	Y = \truM + \truS, 
\end{equation*}
The matrix $\truS$ is assumed to have at most $l_\natural$ many nonzero entries in any column and in any row, and $\truM$ has rank $r_{\natural}$. Moreover, following existing literature we assume that the matrix $\truM$ is  {\em strongly incoherent with parameter $\mu$}. That is, given a singular value decomposition $\truM = \truU \truSig \truV^\top$ with $\truSig\in \real^{\trur \times \trur}$, we let $\mu>0$ denote the smallest constant satisfying
\begin{equation}\label{eq: jointIncoherent}
	\twoinfnorm{\truU}\leq \sqrt{\frac{\mu \trur}{\dmone}}, \qquad \twoinfnorm{\truV}\leq \sqrt{\frac{\mu \trur}{\dmtwo}}, \qquad \text{and}\qquad \infnorm{\truU \truV^\top}\leq \sqrt{\frac{\mu \trur}{\dmone\dmtwo}},
\end{equation}
where $\infnorm{\cdot}$ denotes the entrywise sup-norm.

One common approach for recovering $M_{\natural}$ is to solve the problem:
\begin{equation}\label{eq: minimizationobjectiveRPCA}
	\min_{L,R} \min_{S\in \Omega}\,\fronorm{LR^\top+S-Y}^2.
\end{equation}
where we define the set 
$$\Omega :=\{ S\mid \oneonenorm{S}\leq \oneonenorm{\truS}\},$$
and $\oneonenorm{\cdot}$ is entry-wise $\ell_1$-norm used to promote sparsity. The factors $L$ and $R$ vary over $\RR^{d_1\times k}$ and $\RR^{d_2\times k}$, respectively.
As usual, we focus on the rank overparameterized setting $k\geq r_{\natural}$.
Note that the optimal value of the problem \eqref{eq: minimizationobjectiveRPCA} is clearly zero.

Observe that we may express the problem \eqref{eq: minimizationobjective} more compactly as  
\begin{equation}\label{eqn:key_blah_rpca}
	\min_{L,R}~f(L,R):=\dist^2_{\Omega}(Y-LR^{\top}),
\end{equation}
where $\dist^2_{\Omega}$ denotes the square Frobenius distance to $\Omega$. This is the overparameterized problem that we will focus on.
As usual, we let $\mathcal{S}$ denote the set of minimizers of $f$; note that $\mathcal{S}$ is simply the set of pairs $(L,R)$ satisfying $Y-LR^{\top}\in \Omega$.
Observe that the objective function $f$ is $C^1$-smooth but not $C^2$-smooth. Therefore, in order to measure flatness, we proceed via a smoothing technique introduced in  \cite[Section 4.2]{ha2020equivalence}, \cite[Section 4.3]{ge2017no}. Namely, we approximate $f$ near a basepoint $(\tilde L,\tilde S)$ by the local model:
\begin{equation}\label{eq: minimizationobjective}
	f_{\tilde{L},\tilde{R}}(L,R) :=\fronorm{Y-LR^\top -P_{\Omega}(Y-\tilde{L}\tilde{R}^{\top})}^2,
\end{equation} 
where $P_{\Omega}$ denotes the nearest point projection onto $\Omega$.
It is straightforward to see that the $C^2$-smooth function $f_{\tilde{L},\tilde{R}}(\cdot,\cdot)$ majorizes $f$ and agrees with $f(\cdot,\cdot)$ up to first order at $(\tilde{L},\tilde{R)}$. We may therefore define a minimizer of \eqref{eqn:key_blah_rpca} to be {\em flat} if it solves the problem: 
\begin{equation}\label{eq: rpcaSmallestHessian}
	\min_{(L,R)\in \mathcal{S}}
	\str(D^2f_{L,R}(L,R)).
\end{equation}

The following is the main result of the section.

\begin{thm}[Exact recovery in Robust PCA]
	There is a numerical constant $c>0$ such that in the regime $l _{\natural}\leq \frac{d_{\min}}{\mu \trur}$, any flat minimizer $(L_f,R_f)$ of  \eqref{eqn:key_blah_rpca} satisfies $L_fR_f^\top = \truM$.
\end{thm}

\begin{proof}
	Let $(L_0, R_0)\in \mathcal{S}$ be a solution of \eqref{eqn:key_blah_rpca}. Since $f(L_0,R_0)=0$, the equality 
	$f_{L_0,R_0}(L_0,R_0)=0$ holds. In particular, we may write $f_{L_0,R_0}(L,R)$ as
	$$f_{L_0,R_0}(L,R)=\fronorm{LR^{\top}-W_{\natural}}^2,$$
	where we define $W_{\natural}:=Y-P_{\Omega}(Y-L_0R_0^{\top})$. 
	Therefore appealing to Lemma~\ref{lem:flattest}, we may write the scaled trace as 
	$$\str(D^2f_{L_0,R_0}(L_0,R_0))=2(\fronorm{L_0}^2+\fronorm{R_0}^2).$$
	Thus any flat solution $(L_f,R_f)$ of  \eqref{eqn:key_blah_rpca} solves the problem: 
	\begin{equation}
		\begin{aligned}
			\min_{L,R}&  \quad \fronorm{L}^2 +\fronorm{R}^2\\
			\text{subject to}& \quad Y - LR^\top \in \Omega.
		\end{aligned}
	\end{equation}
	Equivalently, the characterization \eqref{eqn:srebro}  implies that the matrix $X_f = L_fR_f^\top$ solves the problem 
	\begin{equation}\label{eq: rankconsRPCA}
		\begin{aligned}
			\text{minimize}&  \quad \nucnorm{X}\\
			\text{subject to}& \quad Y- X\in \Omega, \\
			& \quad   \rank(X)\leq k.
		\end{aligned}
	\end{equation}
	On the other hand, the result \cite[Theorem 3]{chen2013low}\footnote{The result \cite[Theorem 3]{chen2013low} actually shows that $(\truM,\truS)$ uniquely solves \begin{equation}\label{eq: RPCApenalized}
			\begin{aligned}
				\text{minimize}&  \quad \nucnorm{X} + \lambda \oneonenorm{S}\\
				\text{subject to}& \quad Y = X+ S,
			\end{aligned}
		\end{equation} for some $\lambda >0$. Now for any solution $X_1$ to \eqref{eq: RPCA}, the pair $(X_1,Y-X_1)$ is feasible for \eqref{eq: RPCApenalized} and satisfies 
		$\nucnorm{X_1}+\lambda \oneonenorm{Y-X_1}\leq 
		\nucnorm{\truM}+\lambda \oneonenorm{\truS}$, by definition of $\Omega$. Hence by the uniqueness of \eqref{eq: RPCApenalized}, we know $X_1 =\truM$.
	} shows that $\truM$ is the unique minimizer of the convex relaxation
	\begin{equation}\label{eq: RPCA}
		\begin{aligned}
			\text{minimize}&  \quad \nucnorm{X}\\
			\text{subject to}& \quad Y -X\in \Omega.
		\end{aligned}
	\end{equation}
	Hence, we know $\truM$ also uniquely solves \eqref{eq: rankconsRPCA} and we conclude $\truM = X_f = L_fR_f^\top$, as claimed.
\end{proof}

\section{Neural networks with quadratic activations and covariance matrix estimation}\label{sec: NNQA}
In this section, we investigate flat minimizers of a one hidden layer neural network, considered in the work \cite{soltanolkotabi2018theoretical,li2018algorithmic} for the purpose of analyzing the energy landscape around saddle points. Though this problem is not in the form 
\eqref{eq: main minimization}, we will see that flat minimizers (naturally defined) exactly recover the ground truth under reasonable statistical assumptions. As a special case, we will obtain guarantees for flat minimizers of the overparameterized covariance matrix estimation problem.

The problem setup, following \cite{soltanolkotabi2018theoretical,li2018algorithmic}, is as follows. We suppose that given an input vector $x\in \RR^d$ a response vector $y(x)$ is given by the function
$$y(U_{\natural},x) =v^\top q(\truU ^\top x).$$
We assume that the output weight vector $v\in \real^{r}$ has $r_1$ positive entries and $r_2$ negative entries, the 
hidden layer weight matrix $\truU$ has dimensions $\dm \times r_{\natural}$, and we use a quadratic activation $q(s) = s^2$ applied coordinatewise. 
We get to observe a set of $\ncons$  pairs $(x_i,y_i)\in\RR^d\times\RR$, 
where the features $x_i$ are drawn as $x_i\overset{\text{iid}}{\sim} N(0,I_d)$ and the output values $y_i$ are given by
$$y_i =y(x_i)\qquad \forall i=1,\ldots,m.$$
We aim to fit the data with an overparameterized neural network with a single hidden layer with weights $U\in \real^{d\times k}$ and an output layer with weights $u = (\ones_{k_1},-\ones_{k_2})$, where $k_1\geq r_1$, and $k_2\geq r_2$. 
The prediction $\hat{y}$ of the neural network on input $x$ is thus given by 
\begin{equation}\label{eq: predictionNN}
	\hat{y}(U,x)= u^\top q(U^\top x).
\end{equation} Thus the overparameterized problem we aim to solve is 
\begin{equation}\label{eq: lossfunction}
	\min_{U\in \RR^{d\times k}}~ f(U) := \frac{1}{\ncons} \sum_{i=1}^{\ncons}(\hat{y}(U,x_i)-y_i)^2. 
\end{equation}
As usual, we define the solution set $\mathcal{S}=\{U\in \RR^{d\times k}: f(U)=0\}$. We will see shortly that $\mathcal{S}$ is nonempty and therefore coincides with the set of minimizers of $f$.
Naturally, we declare a matrix $U_f\in \mathcal{S}$ to be {\em flat} if it minimizes the trace of the Hessian of 
$f$ over the set of the minimizers of $f$, i.e., it solves the problem 
$$\min_{U\in \mathcal{S}}~\tr(D^2 f(U)).$$ 
In this section, we aim to show:
\begin{center}
	with high probability over the training set $\{(x_i,y_i)\}_{i=1,\ldots,n}$ flat solutions $U_f$ \\
	achieve zero generalization error, that is 
	$\mathbb{E}_{x\sim N(0,I)}(\hat{y}(U_f,x) -y(U_{\natural},x)) = 0.$
\end{center}

Indeed, we will prove a stronger result by  relating the problem \eqref{eq: lossfunction} to low-rank matrix factorization. To see this, we can write
$\hat{y}(U,x)-y(\truU,x)$ as:
\begin{equation}\label{eq: generalizationErrorWithoutExpectation}
	\begin{aligned}
		\hat{y}(U,x)-y (\truU,x) 
		& = u^\top q(U^\top x) -v ^\top q(\truU ^\top x)\\ 
		& = \inprod{\underbrace{U\diag(u)U^\top}_{U_1U_1^\top - U_2U_2^\top}}{xx^\top} -\inprod{\underbrace{\truU\diag(v)\truU^\top}_{=\,:\truM}}{xx^\top}\\
		& = \inprod{U_1U_1^\top - U_2U_2^\top -\truM}{xx^\top}. 
	\end{aligned}
\end{equation}
Here, we write $U=[U_1, U_2]$ with $U_1\in \real^{\dm \times k_1}$ and $U_2\in \real^{\dm \times k_2}$. Note that the matrix $\truM$ is symmetric. Using \eqref{eq: generalizationErrorWithoutExpectation}, we may rewrite the objective of \eqref{eq: lossfunction} as 
\begin{equation}\label{eqn:funnyNN}
	f(U_1,U_2) = \frac{1}{m}\twonorm{\Amap(U_1U_1^\top - U_2U_2^\top -\truM)}^2,
\end{equation}
where the linear map $\Amap$ is defined as $\Amap:\real^{\dm\times \dm}\rightarrow \real^{\ncons}$ with $[\Amap(Z)]_i = \inprod{Z}{x_ix_i^\top}$ for any $Z\in \real^{\dm\times\dm}$. In particular, from the second equation in \eqref{eq: generalizationErrorWithoutExpectation} and our assumption on $v$, there always exists a matrix $U= [U_1, U_2]$ satisfying $U_1U_1^\top -U_2U_2^\top = \truM$. Therefore, the set of minimizers of $f$ is nonempty and it coincides with $\mathcal{S}$.
Note that in the special case $r_2 = k_2= 0$, the problem \eqref{eqn:funnyNN} becomes covariance matrix estimation  \cite{chen2015exact} and 
further reduces to phase retrieval  when $k_1 = r_1= 1$ \cite{candes2013phaselift}.

Summarizing, the task of  finding a matrix $U$ with a small generalization error, $\abs{\mathbb{E}_{x\sim N(0,I)}[\hat{y}(U,x)-y (\truU,x)]}$,  amounts to implicitly recovering the symmetric matrix $\truM$, but with the parameterization $U_1U_1^\top - U_2U_2^\top$ instead of the usual $LR^\top$ parameterization. The following is the main theorem of the section.

\begin{thm}[Exact recovery]\label{thm: exactRecoverySymmetric}
	There exist numerical constant $c,C>0$, such that in the regime $m\geq C \trur\dm$ and  $\dm \geq C\log m$, with probability at least $1-C\exp(-c\dm)$, any flattest solution $U_f = (U_{f,1},U_{f,2})$ of \eqref{eq: lossfunction} satisfies
	\begin{equation}\label{eq: recoveryErrorNN}
		U_{f,1}U_{f,1}^\top -
		U_{f,2}U_{f,2}^\top = \truM,
	\end{equation}
	and achieves zero generalization error:
	\begin{equation}\label{eq: generalizationError}
		\mathbb{E}_{x\sim N(0,I)}(\hat{y}(U_f,x) -y(\truU,x)) = 0. 
	\end{equation}
\end{thm}

The rest of the section is devoted to the proof of Theorem~\ref{thm: exactRecoverySymmetric}. The general strategy is very similar to the one pursued in Section~\ref{sec:RIP}.
We begin with the following lemma that expresses the trace of the Hessian in the same spirit as Lemma~\ref{lem: scaledtraceFronormSquareSum}. With this in mind, we define the  matrix 
\begin{equation}\label{eq: Ddef}
	D:= \left(\frac{1}{\ncons \dm} \sum_{i=1}^m A_i A_i^\top\right)^{\frac{1}{2}}.
\end{equation}
\begin{lemma}[Trace]\label{lem: traceFronormSquareSum}
	The second order derivative of the function $f$ at any matrix $[U_1,U_2]\in \mathcal{S}$ is the quadratic form:
	$$D^2 f(U_1,U_2)(V_1,V_2)=\frac{2}{m}\|\mathcal{A}(U_1V_1^{\top}+V_1U_1^{\top}-U_2V_2^{\top}-V_2U_2^{\top})\|_2^2$$
	Moreover, the trace can be written as 
	\begin{equation}
		\tr(D^2f(U_1,U_2)) 
		= 4d\fronorm{DU_1}^2 + 4d\fronorm{D U_2}^2.
	\end{equation}
\end{lemma}
\begin{proof}
	The expression for $D^2 f(U_1,U_2)[V_1,V_2]$ follows immediately from algebraic manipulations.	The trace of the Hessian therefore can be written as 
	\begin{equation}
		\begin{aligned} 
			&\tr(D^2f(U_1,U_2)) \\ 
			= &\frac{2}{\ncons }\left( \sum_{i=1}^{d} \sum_{j=1}^k\twonorm{\Amap( U _1 e_j e_i^\top + e_ie_j^\top U_1^\top  )}^2\right)+ \frac{2}{\ncons }\left(\sum_{i=1}^{d} \sum_{j=1}^k \twonorm{\Amap(e_i e_j^\top U_2^\top 
				+ U_2e_i^\top e_j
				)}^2\right).
		\end{aligned}
	\end{equation}
	Using the symmetry of the matrices $A_i$, the first term can be written as 
	\begin{equation}
		\begin{aligned}
			\sum_{i=1}^{d} \sum_{j=1}^k\twonorm{\Amap( U _1 e_j e_i^\top + e_ie_j^\top U_1^\top  )}^2  & = 
			2 \sum_{i=1}^{d} \sum_{j=1}^k\sum_{l=1}^m 
			\inprod{A_l}{U_1e_j e_i^\top}^2  .
		\end{aligned}
	\end{equation}
	Following exactly the same computation as \eqref{eqn:need_later} completes the proof.	
\end{proof}

In particular, Lemma~\ref{lem: traceFronormSquareSum} implies that flat solutions are exactly the minimizers of the problem 
\begin{equation}\label{eq: scaled trace Sum of Square_NN}
	\min_{U_1,U_2} ~
	\frac{1}{2}\left(\fronorm{DU_1}^2 + \fronorm{DU_2}^2\right)\qquad \textrm{such that}\qquad  \Amap(U_1U_1^\top - U_2U_2^\top)=\Amap(\truM).
\end{equation}
We would like to next rewrite this problem in terms of minimizing a nuclear norm of a $d\times d$ matrix. With this in mind, we will require the following two lemmas in the spirit of the characterization of the nuclear norm \eqref{eqn:srebro}.

\begin{lemma}[Decompositions and pos/neg eigenvalues]\label{lem:decomp_NN}
	The following two statements are equivalent for any symmetric matrix $X\in \RR^{d\times d}$.
	\begin{enumerate}
		\item\label{it:1syms} $X$ admits a decomposition $X=  U_1 U_1^\top- U_2  U_2^\top$ for some matrices $ U_i \in \real^{\dm\times k_i}$,
		\item\label{it:2syms} $X$ has at most $k_1$ non-negative eigenvalues and $k_2$ non-positive eigenvalues.
	\end{enumerate}
\end{lemma}
\begin{proof}
	The implication $\ref{it:2syms}\Rightarrow\ref{it:1syms}$ follows immediately from an eigenvalue decomposition of $X$. Conversely, suppose that \ref{it:1syms} holds.
	Observe that \ref{it:1syms} clearly is equivalent to being able to write $X = A-B$ with $A,B\succeq 0$, $\rank(A)\leq k_1$, and $\rank(B)\leq k_2$. Let $r_1$ be the number of strictly positive eigenvalues of $X$ and let $r_2$ be the number of strictly negative eigenvalues of $X$. We now prove $r_1\leq k_1$ by contradiction. A similar arguments yields $r_2\leq k_2$. Suppose indeed $r_1>k_1$ and consider the matrix $A := X+B$. Let $\mathcal{U}$ be the span of eigenspaces corresponding to the top $r_1$ eigenvalues of $X$. Note that $\mathcal{U}$ has dimension $r_1$. 
	Cauchy's interlacing theorem implies that the $r_1$-th largest eigenvalue of 
	$X+B$ satisfies that 
	$\lambda_{r_1}(X+B)\geq \min_{v\in \mathcal{U}\setminus\{0\}}\frac{v^\top (X+B)v}{v^\top v}$. Since $B\succeq 0$, for any $v\in \mathcal{U}\setminus\{0\}$ we estimate $\frac{v^\top (X+B)v}{v^\top v}= \frac{v^\top Xv+v^\top Bv}{v^\top v}\geq \frac{v^\top Xv}{v^\top v}\geq  \lambda_{r_1}(X)>0$. We conclude that that rank of $A$ is at least $r_1$, which is a contradiction since $A$ has rank at most $k_1$.
\end{proof}

\begin{lemma}\label{lem: symNucnorm}
	If a symmetric matrix $X\in \RR^{d\times d}$ admits a decomposition $X= \bar U_1\bar U_1^\top-\bar U_2 \bar U_2^\top$ for some matrices $\bar U_i \in \real^{\dm\times k_i}$, then equality holds:
	$$\nucnorm{X} = 
	\min_{X= U_1U_1^\top-U_2U_2^\top, U_i\in \real^{\dm \times k_i}} \fronorm{U_1}^2+\fronorm{U_2}^2.$$
\end{lemma}
\begin{proof}
	First, using triangle inequality, for any $(U_1,U_2)$ such that $X= U_1U_1^\top-U_2U_2^\top$, we have 
	$\nucnorm{X}= \nucnorm{U_1U_1^\top-U_2U_2^\top}\leq\nucnorm{U_1U_1^\top}+\nucnorm{U_2U_2^\top} = \tr(U_1U_1^\top)+\tr(U_2U_2^\top) =  \fronorm{U_1}^2+\fronorm{U_2}^2$.
	Conversely, suppose that we may write $X= U_1U_1^\top-U_2U_2^\top$ for some $U_1,U_2$. Then Lemma~\ref{lem:decomp_NN} implies that $X$ has at most $k_1$ non-negative eigenvalues and 
	$k_2$ non-positive eigenvalues. Thus, we may write any eigenvalue decomposition of $X$ as $X = P_1\Lambda_1P_1^\top - P_2\Lambda_2P_2^\top$, where the diagonal matrices $\Lambda_i\in \RR^{r_i \times r_i}$ have positive entries . By taking $U_i = [P_i \sqrt{\Lambda_i},0_{k_i-r_i}]$, we see that equality $\nucnorm{X} = \fronorm{U_1}^2+\fronorm{U_2}^2$ holds and the proof is complete.
\end{proof}

Lemma~\ref{lem:decomp_NN} and ~\ref{lem: symNucnorm} directly imply that the problem \eqref{eq: scaled trace Sum of Square_NN}, which characterizes flat solutions, is equivalent to the rank constrained problem:
\begin{equation}\label{eq: PerturbnucnormXsymmetric}
	\begin{aligned}
		\text{minimize} & \quad \nucnorm{X}. \\ 
		\text{subject to} & \quad \Amap(D^{-1}XD^{-1})= \Amap(\truM),\\
		&\quad \text{$X$ is symmetric and has at most $k_1$ positive eigenvalues and $k_2$ negative eigenvalues}.
	\end{aligned}
\end{equation}
Therefore a natural convex relaxation simply drops the requirements on the eigenvalues:
\begin{equation}\label{eq: PerturbnucnormXsymmetric_convex}
	\begin{aligned}
		\text{minimize} & \quad \nucnorm{X} \\ 
		\text{subject to} & \quad \Amap(D^{-1}XD^{-1})= \Amap(\truM),\\
		&\quad \text{$X\in\RR^{d\times d}$ is symmetric}.
	\end{aligned}
\end{equation}
The following theorem summarizes these observations.

\begin{thm}[Convex relaxation]\label{eq: exactConvexSymmetric}
	Suppose that the matrix $D$ is invertible. Then the problems \eqref{eq: scaled trace Sum of Square_NN} and \eqref{eq: PerturbnucnormXsymmetric} are equivalent in the following sense. 
	\begin{enumerate}
		\item The optimal values of \eqref{eq: scaled trace Sum of Square_NN} and \eqref{eq: PerturbnucnormXsymmetric} are equal.
		\item If $[U_1, U_2]$ solves  \eqref{eq: scaled trace Sum of Square_NN}, then 
		$X=D(U_1U_1^{\top}-U_2U_2^{\top})D$ is optimal for  \eqref{eq: PerturbnucnormXsymmetric}.
		\item If a minimizer $X$ of \eqref{eq: PerturbnucnormXsymmetric} has an eigenvalue decomposition $X=P_1\Lambda_1 P^{\top}_1-P_2\Lambda_2 P^{\top}_2$ for some diagonal matrices $\Lambda_i\in \RR^{r_i\times r_i}$ with positive entries, then the matrices $U_i=[D^{-1}P_i\sqrt{\Lambda_i}, 0_{d,(k_i-r_i)}]$, $i=1,2$  are minimizers of \eqref{eq: scaled trace Sum of Square_NN}.
	\end{enumerate}
	Moreover, if $X=DM_{\natural}D$ is the unique minimizer of  the problem \eqref{eq: PerturbnucnormXsymmetric_convex}, then any flat solution $[U_1,U_2]$ satisfies $U_1U_1^\top-U_2U_2^\top = M_{\natural}$. 
\end{thm}

Next, we aim to understand when the problem \eqref{eq: PerturbnucnormXsymmetric_convex} exactly recovers $D^{-1}M_{\natural}D$. The difficulty is that even in the case $k_2=0$, corresponding to covariance matrix estimation, the linear map $\mathcal{A}$ satisfies $\ell_1/\ell_2$ RIP only if $m\succsim r_{\natural}^2d$ \cite{chen2015exact}, which is suboptimal by a factor of $r_{\natural}$.
We will sidestep this issue by relating the program \eqref{eq: PerturbnucnormXsymmetric} to one with a different linear map that does satisfy the $\ell_1/\ell_2$-RIP condition over all rank $r_{\natural}$ matrices in the optimal regime $m\succsim r_{\natural}d$. The reduction we use is inspired by \cite[Equation (0.36)]{cai2015ropSupp}.
\begin{lemma}\label{lem: A1andAsolution}
	Define the linear map $\Amap_1: \real^{\dm \times \dm} \rightarrow \real^{\lfloor \ncons \rfloor/2}$ by
	\begin{equation}\label{eq: A1def}
		[\Amap_1(Z)]_i = \left(\frac{x_{2i-1}+x_{2i}}{\sqrt{2}} \right)^\top Z \left(\frac{x_{2i-1}-x_{2i}}{\sqrt{2}} \right), \qquad\textrm{for } i=1,\dots, \lfloor \ncons \rfloor/2,
	\end{equation}
	and consider the convex optimization program 
	\begin{equation}\label{eq: PerturbnucnormXA1}
		\min_{\Amap_1(D^{-1}XD^{-1})=\Amap_1(\truM)} \nucnorm{X}. 
	\end{equation}
	If $D\truM D$ is the unique solution of \eqref{eq: PerturbnucnormXA1}, then it is also the unique solution of \eqref{eq: PerturbnucnormXsymmetric_convex}.
\end{lemma}
\begin{proof}
	We first note the equalities 
	\begin{equation}
		\begin{aligned} 
			[\Amap_1(Z)]_i & = \left(\frac{x_{2i-1}+x_{2i}}{\sqrt{2}} \right)^\top Z \left(\frac{x_{2i-1}-x_{2i}}{\sqrt{2}} \right) \\ 
			& = \frac{1}{2}\inprod{x_{2i-1}x_{2i-1}^\top}{Z} - 
			\frac{1}{2}\inprod{x_{2i}x_{2i}^\top}{Z} \\
			& = \frac{1}{2} \left([\Amap_{2i-1}(Z)] - [\Amap_{2i}(Z)]\right).
		\end{aligned} 
	\end{equation}
	It follows immediately that any $X$ that is feasible for \eqref{eq: PerturbnucnormXsymmetric_convex} is also feasible for \eqref{eq: PerturbnucnormXA1}.
	Consequently, if the symmetric matrix $D\truM D$ is a unique minimizer of  \eqref{eq: PerturbnucnormXA1}, then it must also be a unique minimizer of \eqref{eq: PerturbnucnormXsymmetric_convex}. This completes the proof.	
\end{proof}

We are now ready to prove Theorem \ref{thm: exactRecoverySymmetric}. 
\begin{proof}[Proof of Theorem \ref{thm: exactRecoverySymmetric}]
	Notice that the two vectors, $\frac{x_{2i-1}+x_{2i}}{\sqrt{2}}$ and $\frac{x_{2i-1}-x_{2i}}{\sqrt{2}}$, are jointly normal and uncorrelated, and therefore are independent. Consequently, we see that the map $\Amap_1$ defined in  \eqref{eq: A1def} follows the bilinear sensing model (Definition~\ref{defn:mat_sense_bilin}). Therefore Lemma~\ref{lem:RIP for bilinear sensing} implies that for any positive integer $k\geq 2$, there exist constants $c,C>0$ depending only on $k$ and numerical constants $\delta_1,\delta_2>0$ and  such that in the regime $m\geq cr_{\natural} d$, with probability at least $1-\exp(-C m)$, the measurement map $\mathcal{A}_1$ satisfies $\ell_1/\ell_2$ RIP with parameters $(k r_{\natural}, \delta_1, \delta_2)$. In this event, Lemma \ref{lem:reparam} with   $Q_1=Q_2=D^{-1}$ implies that $\mathcal{A}_1(D^{-1}\cdot D^{-1})$ satisfies $\ell_1/\ell_2$ RIP with parameters $(kr_{\natural},\alpha_2^{-2} \delta_1, \alpha_1^{-2}\delta_2)$, where $\alpha_1>0$ is a lower bound on the minimal eigenvalue of $D$ and  $\alpha_2>0$ is an upper-bound on the maximal eigenvalue of $D$. In order to estimate $\kappa$, we may write
	$$D^2=\frac{1}{md}\sum_{i=1}^{m}A_iA_i^{\top}=\frac{1}{md}\sum_{i=1}^{m}\|x_i\|^2 x_ix_i^{\top}.$$
	Exactly the same proof as that of Lemma~\ref{lem_condbilin} ensures that there exist constants $c_1,c_2,c_3,c_4>0$ such that as long as we are in the regime, $m\geq c_3d$ and $\log(m)\leq c_4d$, 
	the estimate holds:
	$$\mathbb{P}\left\{\frac{1}{2} I_{d}\preceq D\preceq \frac{3}{2}I_{d}\right\}\geq 1-c_2e^{-c_1d}.$$
	Consequently, in this regime we may upper bound the condition number $\kappa$ of $D$  by $3$. In light of Lemma~\ref{lem: exactrecoverConvex}, in order to ensure that $D\truM D$ is the unique minimizer of \eqref{eq: PerturbnucnormXA1}, it remains to simply choose $k$ such that the inequality $\frac{9\delta_2}{\delta_1} \leq \sqrt{k}$ holds. Using Lemmas~\ref{eq: exactConvexSymmetric}-\ref{lem: A1andAsolution} completes the proof.
\end{proof}

\section{Numerical experiments}\label{sec: num}
Recall that we have proved that for a variety of overparameterized problems, under standard statistical assumptions, in the noiseless setting, (1) flat solutions recover the ground truth and (2) flat solutions are nearly  norm-minimal and nearly-balanced (but not exactly). 
In this section, we numerically validate both of the claims, in order. 
Note that finding flat solutions in these examples, amounts to solving a convex optimization problem as long as the number of measurements is sufficiently large.

\begin{figure}
	\centering
	\begin{subfigure}[b]{0.35\textwidth}
		\centering
		\includegraphics[width=\textwidth]{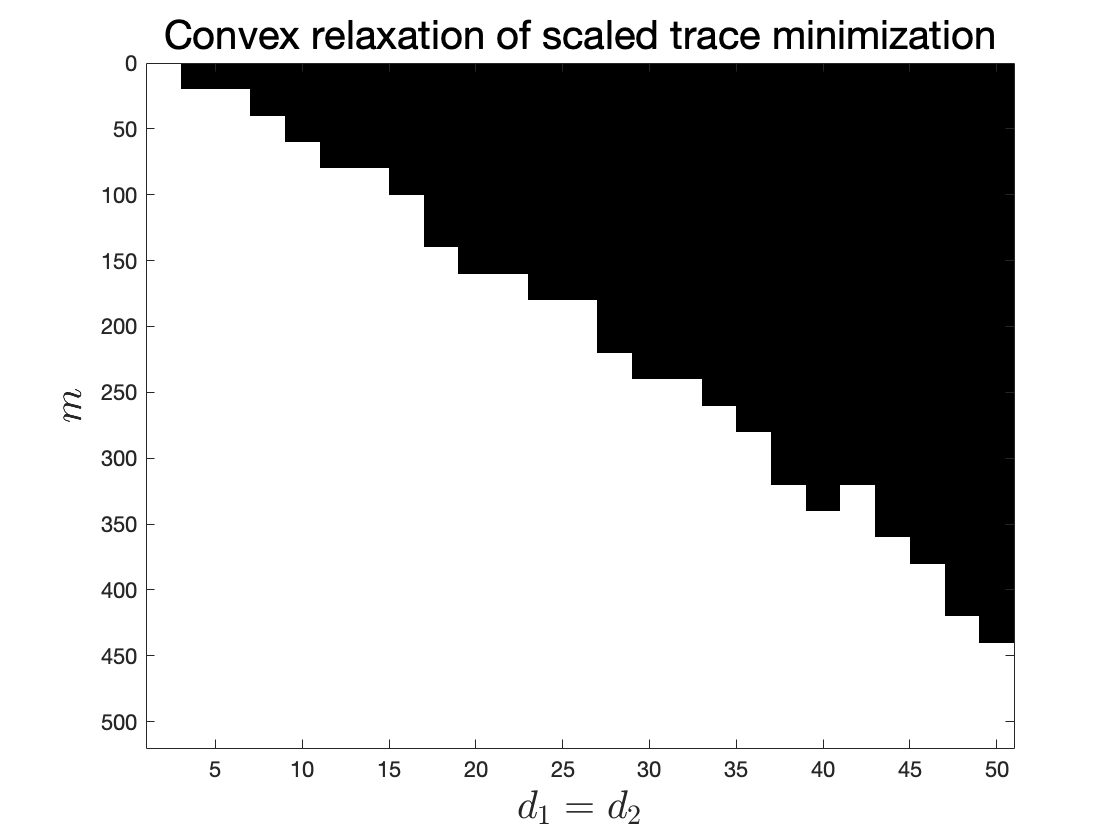}
		\includegraphics[width= \textwidth]{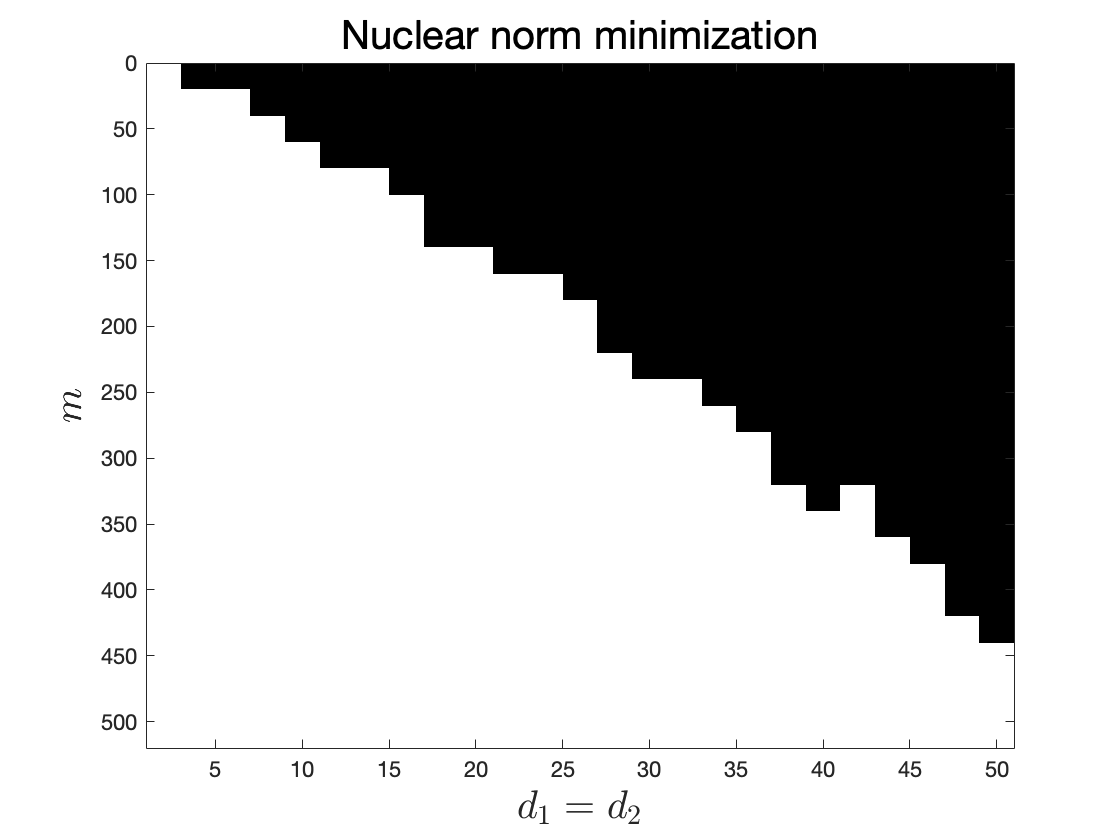}
		\caption{Matrix Sensing     \label{fig: MS}}
	\end{subfigure}
	\qquad
	\begin{subfigure}[b]{0.35\textwidth}
		\centering
		\includegraphics[width=\textwidth]{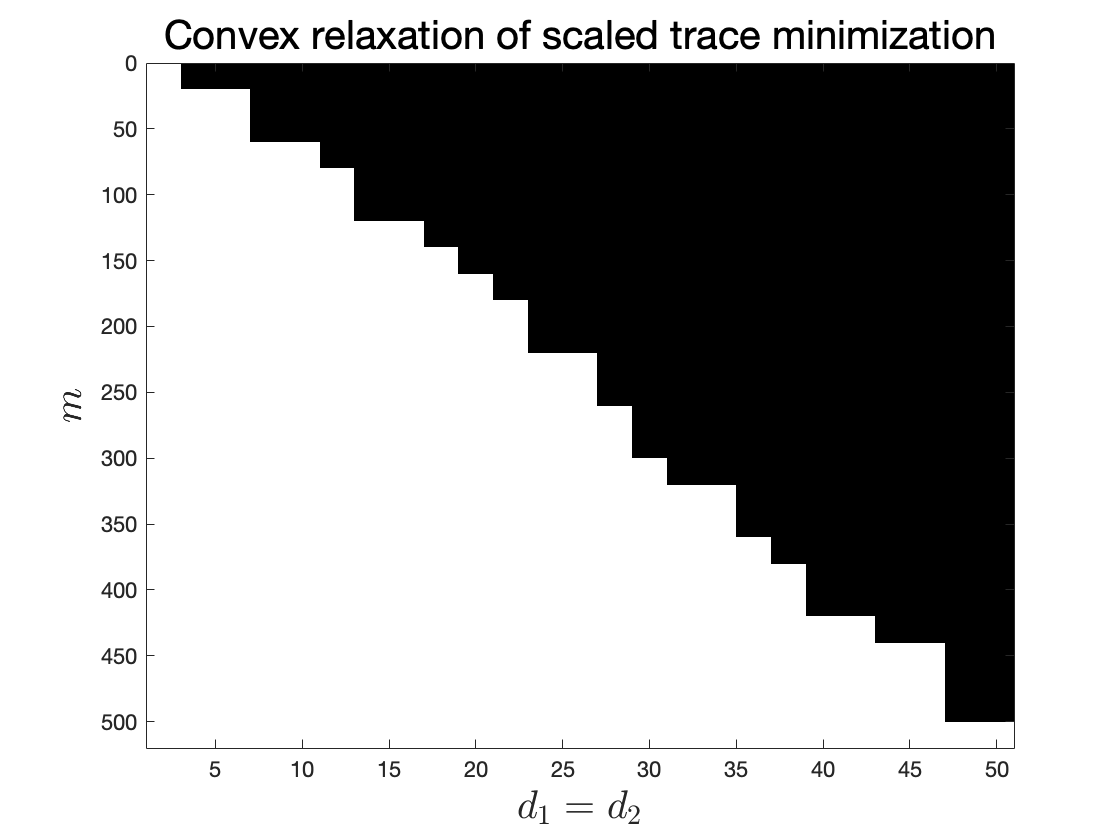}
		\includegraphics[width= \textwidth]{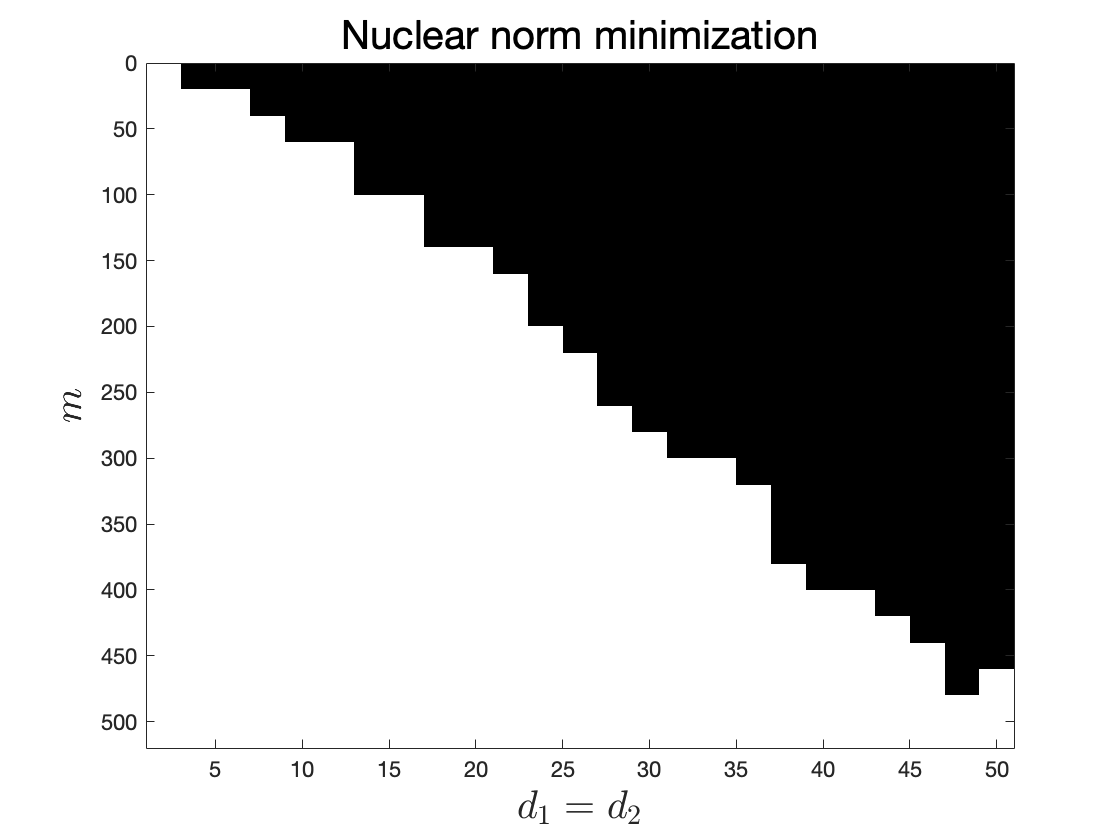}
		\caption{Bilinear Sensing   \label{fig: BS}}
	\end{subfigure}
	\begin{subfigure}[b]{0.35\textwidth}
		\centering
		\includegraphics[width=\textwidth]{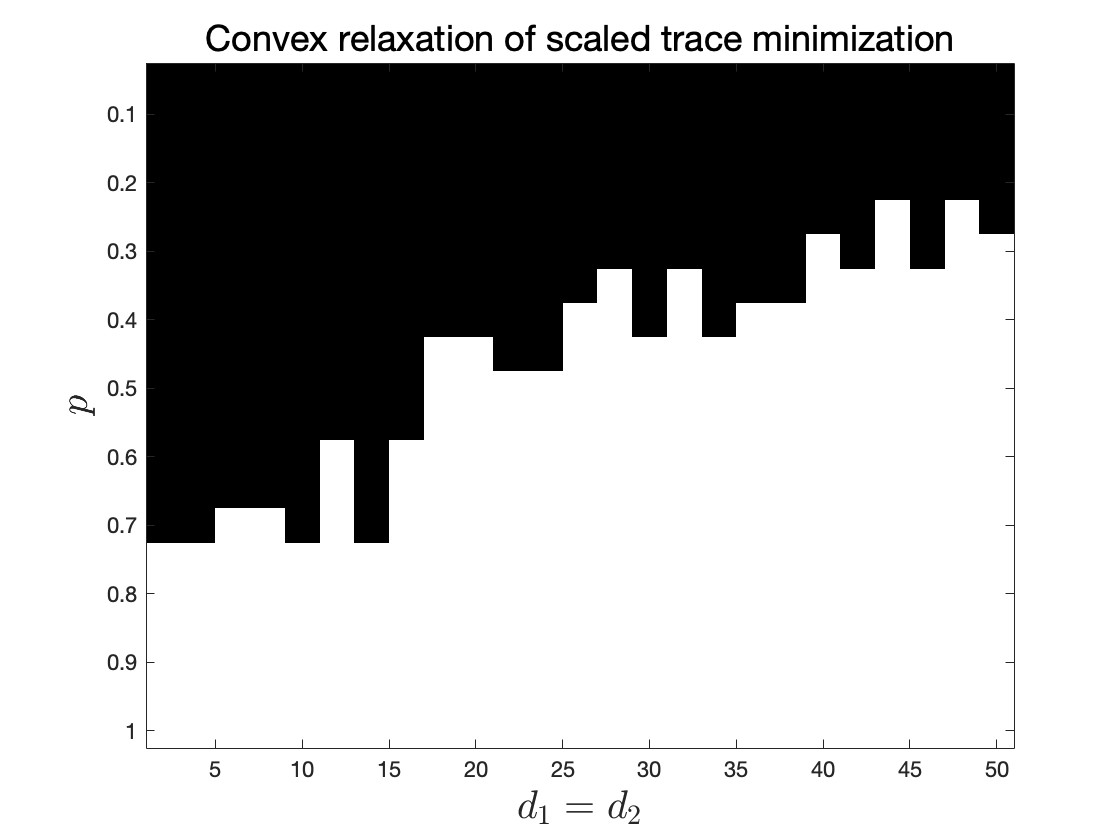}
		\includegraphics[width= \textwidth]{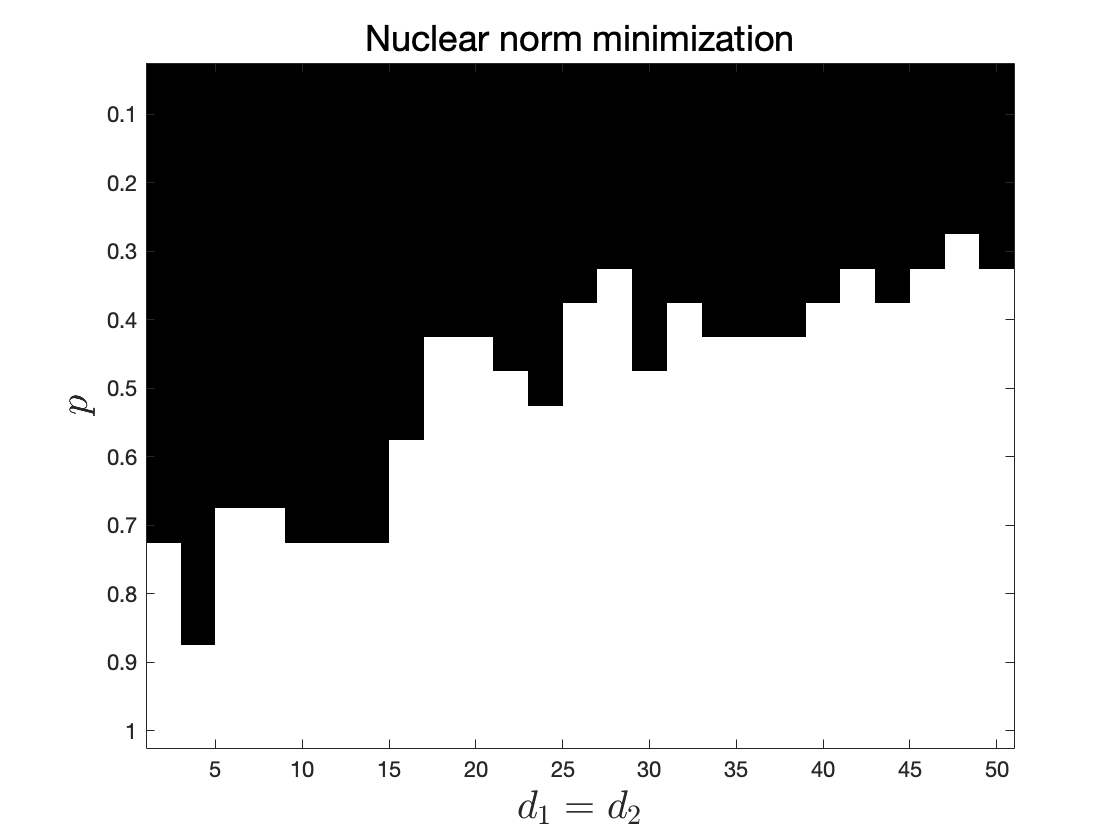}
		\caption{Matrix Completion     \label{fig: MC}}
	\end{subfigure}
	\qquad
	\begin{subfigure}[b]{0.35\textwidth}
		\centering
		\includegraphics[width=\textwidth]{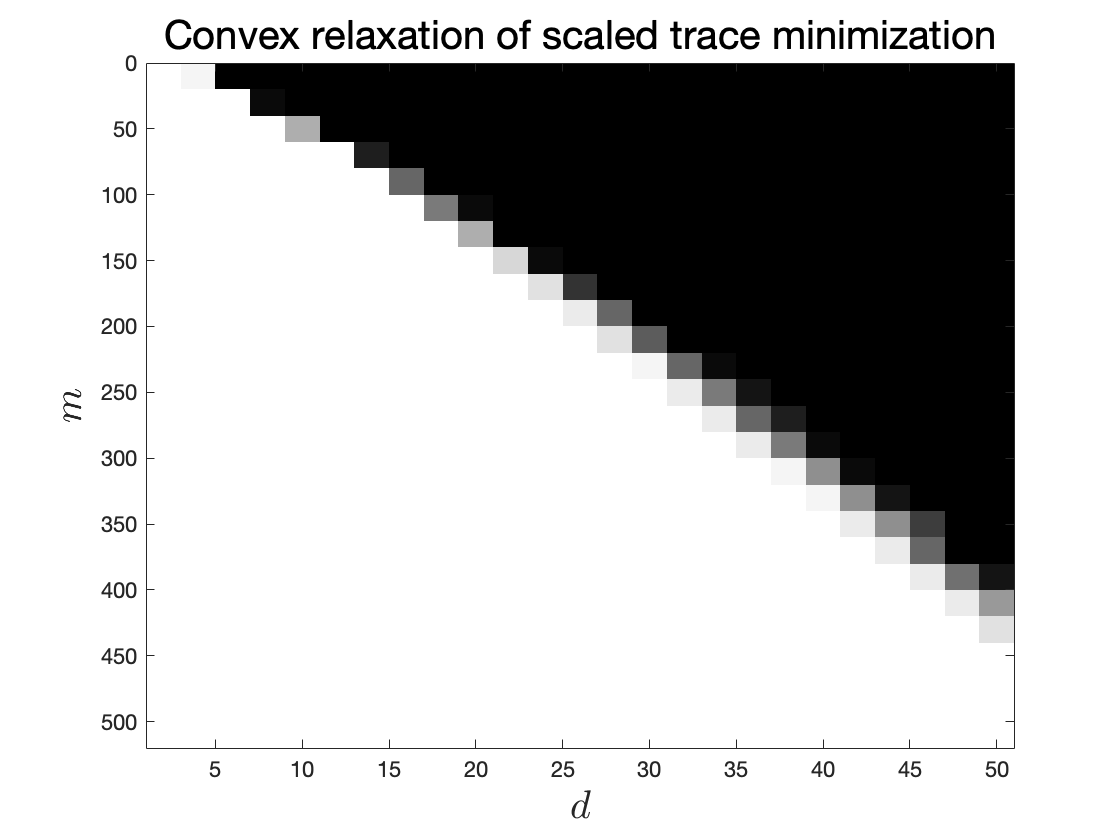}
		\includegraphics[width= \textwidth]{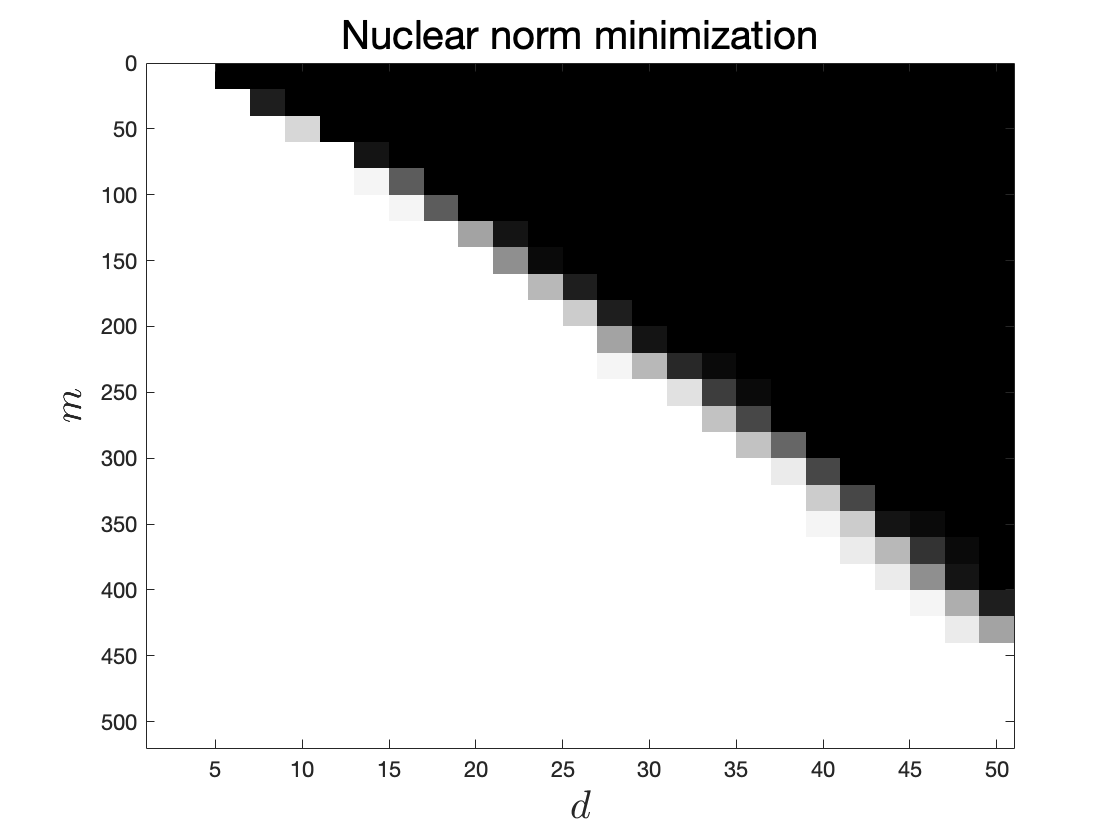}
		\caption{NN with quadratic activation  \label{fig: NNQ}}
	\end{subfigure}
	\caption{The empirical probability of successful recovery of $\truM$ for different combination of dimension $d$ and number of measurements $m$ ($p$ for matrix completion). We use gray scale and the whiter the color, the higher probability of success.}
	\label{fig: exact recovery}
\end{figure}

\paragraph{Experiment setup.}  We consider four problems described earlier in the paper: (a) matrix sensing, (b) bilinear sensing, (c) matrix completion, and (d) neural networks with quadratic activation. For each setting, we consider different combination of the dimension $d = \dmone = \dmtwo$ and the number of measurements $m$ ($p$ for matrix completion). For each combination $(\dm, \ncons)$ ( $(\dm,p)$ for matrix completion), we randomly generate a rank $2$ ground truth unit Frobenius norm matrix $\truM$ (rank $3$ for the setting of neural network with quadratic activation), then repeatedly generate the linear measurement map $\Amap$ and solve ten times the convex relaxation associated with being a flat solution and the nuclear norm minimization problem.

\paragraph{Exact Recovery.} To measure the success of exact recovery, for a solution $\hat{X}$ from the convex relaxation of the scaled trace problem (or from the nuclear norm minimization), we measure the Frobenius norm error $\fronorm{D_1^{-1}\hat{X}D_2^{-1}-\truM}$ (or $\fronorm{\hat{X}-\truM}$ for the nuclear norm minimization). Our criterion for exact recovery is whether this error is smaller than $10^{-6}$ or not. Figure \ref{fig: exact recovery} shows the empirical probability of success recovery (averaging over ten times) for each combination of dimension and number of measurements. The figure is in gray scale and the whiter color indicates higher success probability. We observe that the frequency of exact recovery by flat solutions almost matches the frequency of exact recovery by nuclear norm minimization. Notice moreover that flat solutions exactly recover the ground truth matrix, though we are only able to show weak recovery for matrix completion.

\paragraph{Regularity.} Next we test the regularity of flat solutions for the (a) matrix sensing, (b) bilinear sensing, (c) matrix completion problems. We only consider the pairs $(\dm,\ncons)$ such that the 
matrices $D_1, D_2$ are nonsingular.  Let $\hat{X}$ be the solution of the convex relaxation for being a flat solution and let $\hat{X}_{\texttt{nuc}}$ be the solution to the nuclear norm minimization problem.
We compute the factors $L_f = D_1^{-1}U\sqrt{\Sigma}$ and $R_f = D_2^{-1}V\sqrt{\Sigma}$ using the full SVD of $\hat{X}= U\Sigma V^\top$. We then use the quantity $\frac{\fronorm{L_f}^2+\fronorm{R_f}^2}{2\nucnorm{\hat{X}_{\texttt{nuc}}}}$ to measure the norm-minimality of flat solutions, and the quantity $\nucnorm{L_f^\top L_f - R_f^\top R_f}/\nucnorm{\truM}$ to measure balancedness. Whenever one of the matrices $D_1,D_2$ is singular, we set both measures to be $10^{20}$. The result (in $\log10$ scale) is shown in Figure \ref{fig: regularity}. We observe that whenever flat solutions exactly recover the ground truth, both measures are small but not exactly zero. In particular, the norm-minimal and flat solutions are distinct. 

\begin{figure}[H]
	\centering
	\begin{subfigure}[b]{0.48\textwidth}
		\centering
		\includegraphics[width=\textwidth]{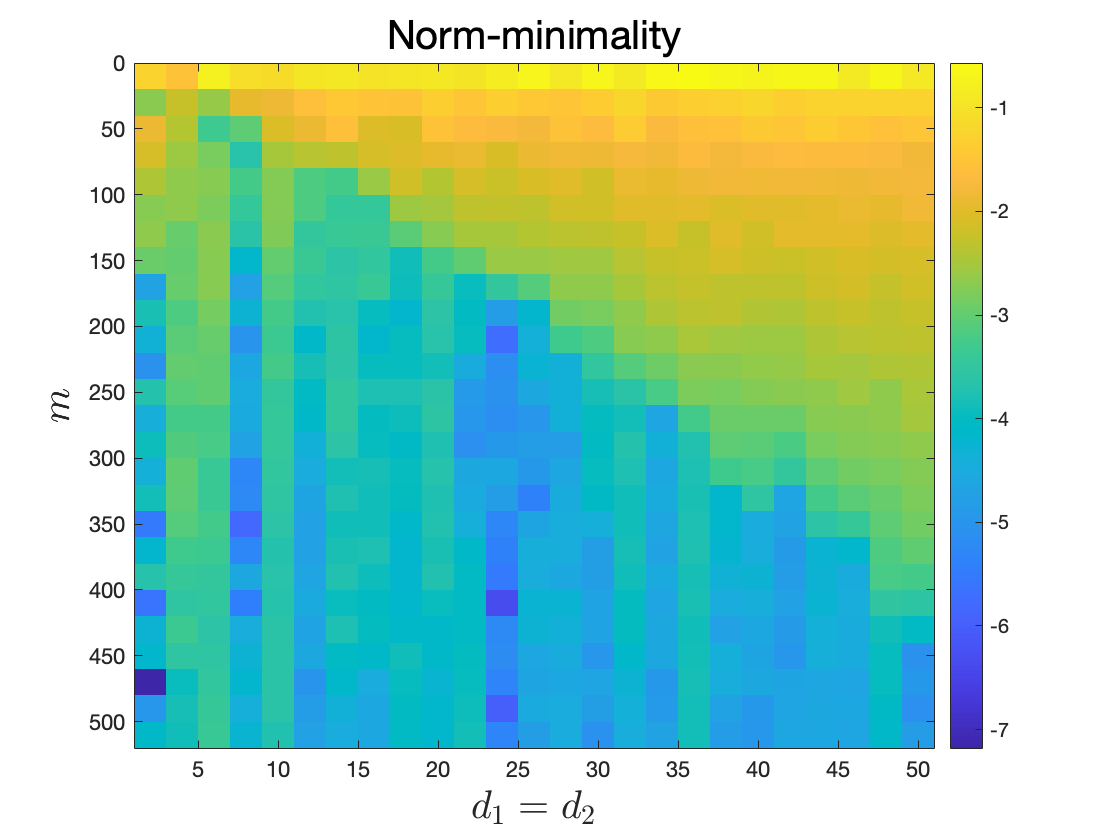}
		\includegraphics[width= \textwidth]{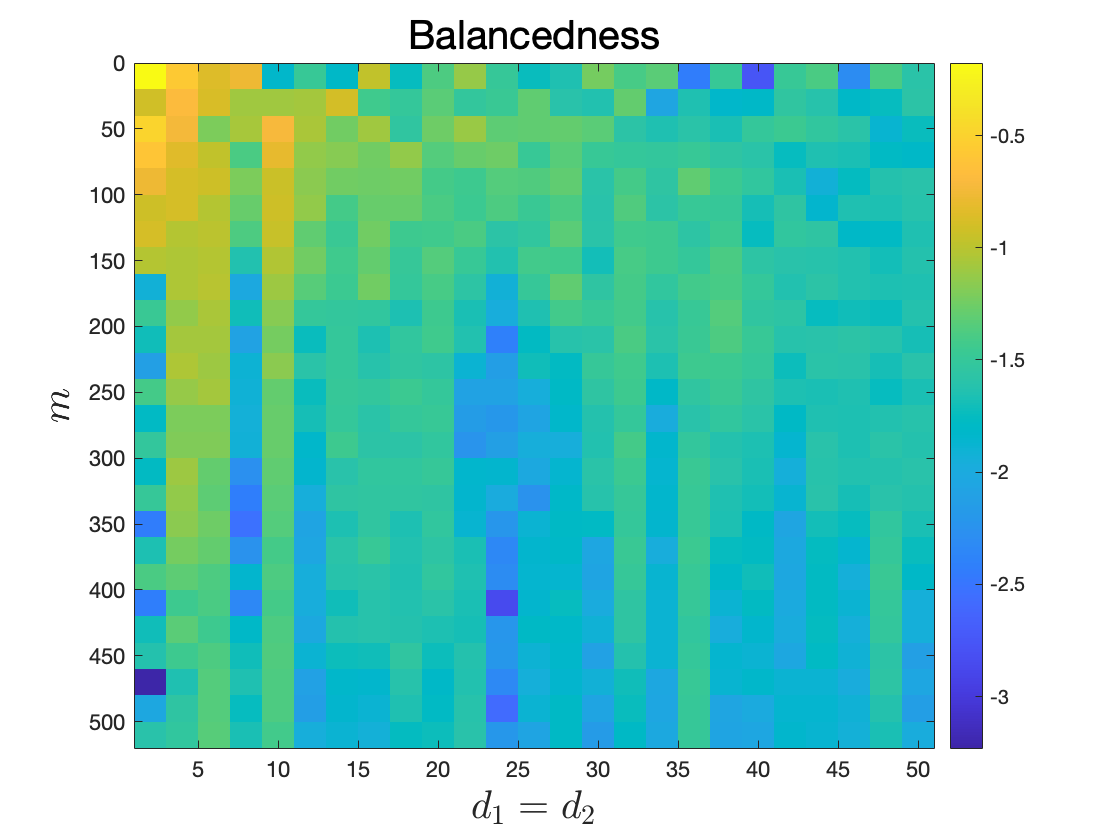}
		\caption{Matrix Sensing   \label{fig: MSreg}}
	\end{subfigure}
	\hfill
	\begin{subfigure}[b]{0.48\textwidth}
		\centering
		\includegraphics[width=\textwidth]{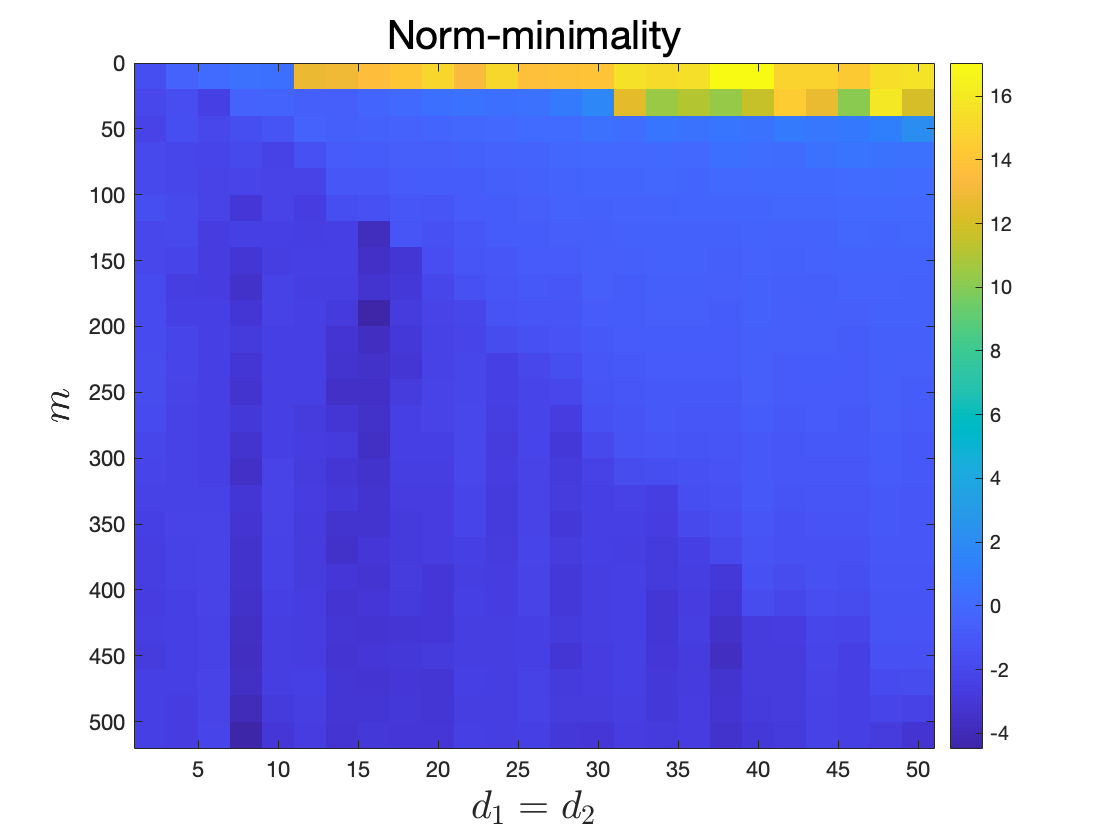}
		\includegraphics[width= \textwidth]{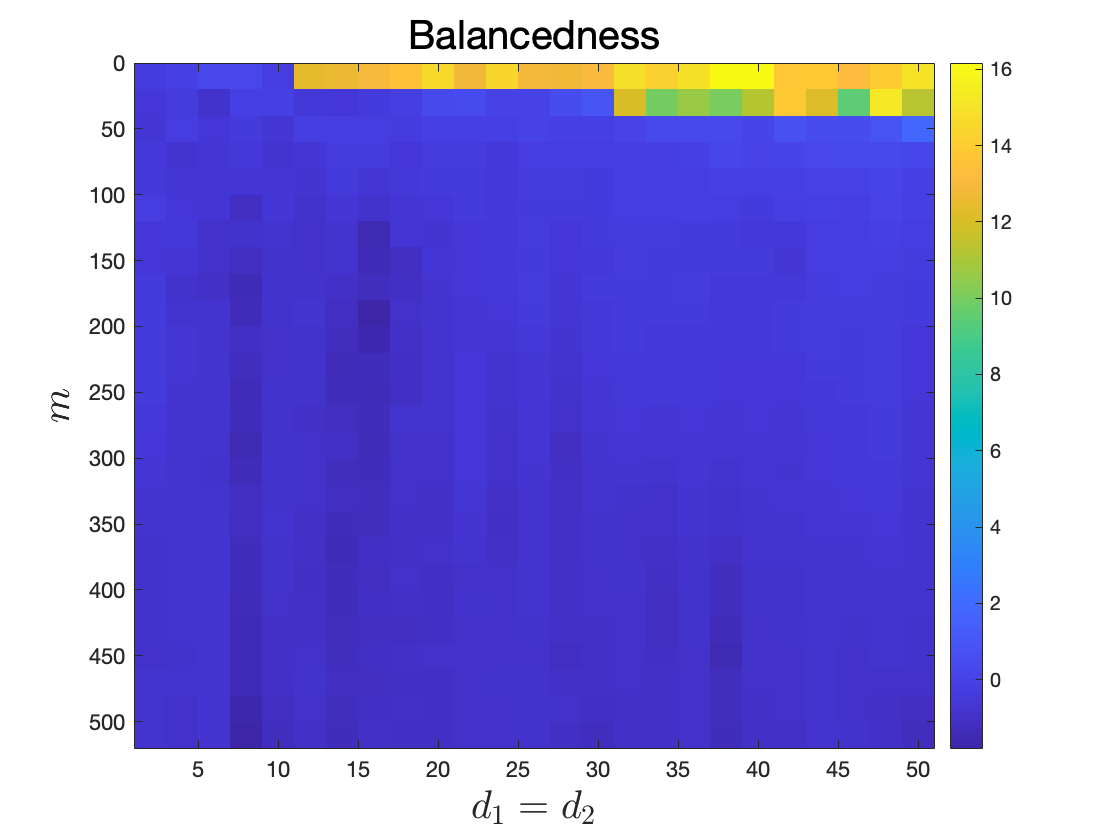}
		\caption{Bilinear Sensing \label{fig: BSreg}}
	\end{subfigure}
	\caption{The empirical average (over $10$ trials) of the regularity measure : (1) minimal norm, $\frac{\fronorm{L_f}^2 + \fronorm{R_f}^2}{2 \nucnorm{\hat{X}_{\texttt{cvx}}}}-1$, and (2) balance, $\frac{\nucnorm{L_f^\top L_f - R_f^\top R_f}}{\nucnorm{\truM}}$. Both measures are presented in $\log10$ scales. }
	\label{fig: regularity}
\end{figure}


\input{conclusion_and_discussion}
\bibliographystyle{plain}
\bibliography{reference}

\appendix 
\input{extension_to_noisy_observation}

\end{document}

%% file: conclusion_and_discussion.tex
\section{Conclusion and discussion on depth}\label{sec: Effect_of_depth}

In this paper, we analyzed a variety of low rank matrix recovery problems in  rank-overparameterized settings. We considered overparameterized matrix and bilinear sensing, robust PCA, covariance matrix estimation, and single hidden layer neural networks with quadratic activation functions. In all cases, we showed that flat minima, measured by the scaled trace of the Hessian, {\em exactly recover} the ground truth under standard statistical assumptions. For matrix completion, we established weak recovery, although empirical evidence suggests exact recovery holds here as well.  

Matrix factorization problems are suggestive of the behavior one may expect for two layer neural networks. Therefore, an appealing question is to consider the effect that depth may have on generalization properties of flat solutions. In this section, we argue that depth may not bode well for generalization of flat solutions. 
As a simple model, we consider the setting of sparse recovery under a ``deep'' overparameterization. Namely, consider a ground truth vector $\trux\in \real^{\dm}$ with at most $\trur$ nonzero coordinates.  The goal is to recover $\trux$ from the observed measurements $b=A\trux$ under a linear map $A\colon\real^{\dm} \rightarrow \real^{\ncons}$. We assume that $A$ satisfies the restricted isometry property (RIP): there exist $(\delta_1,\delta_2)$ such that 
\begin{equation}\label{eq: vectorRIP}
\delta_1 \|x\|_2\leq \tfrac{1}{\sqrt{m}}\|Ax\|_2\leq \delta_2\|x\|_2,
\end{equation}
for all $x\in \RR^d$ that have at most $2r$ nonzero coordinates. The simple least square formulation for finding consistent dense signals is 
\begin{equation}\label{eq: ls}
	\min_{x\in \real^\dm} \quad g(x):=\frac{1}{m} \twonorm{Ax -b}^2.
\end{equation}
We  introduce overparameterization by parameterizing the variable $x$ as the Hadamard product $\odot$ of $k$ factors $x= v_1\odot v_2\odot \dots \odot v_k$ with $v_i\in \RR^d$. Thus, the problem \eqref{eq: ls} becomes 
\begin{equation}\label{eq: overparametrization}
	\min_{v_1,\dots,v_k} \quad f(v):=\frac{1}{m} \twonorm{A(v_1\odot \dots \odot v_k ) -b}^2\qquad \textrm{with }v_i \in \real^{\dm},\;i=1,\dots,k.
\end{equation}
The flat solutions are naturally defined as those $(v_i)_{i=1}^k$ solving the following problem: 
\begin{equation}\label{def: flat_solution_sparse_vector_recovery}
	\min_{v_i\in \real^\dm,i=1,\dots,k} \tr(D^2 f(v_1,\dots,v_k)) \quad 
	\text{subject to}\quad A(v_1\odot \dots \odot v_k) =b.
\end{equation}
To compute the Hessian $\tr(D^2 f(v_1,\dots,v_k))$, let $a_i$ be the $i$-th column of $A$. Following a similar calculation as in Lemma \ref{lem: scaledtraceFronormSquareSum} yields the expression
\begin{equation}\label{eqn: trace}
\tr(D^2 f([v_i]_{i=1}^k))  = \sum_{i=1}^k \twonorm{\sqrt{D} (v_1\odot \dots v_{i-1}\odot v_{i+1}\dots \odot v_k)}^2, 
\end{equation}
for any $(v_i)_{i=1}^k \in \real^{\dm \times k}$ where 
\begin{equation}
	D := \frac{1}{m}\diag(a_1^\top a_1,\dots, a_\dm ^\top a_\dm ).
\end{equation} 
The following lemma shows that $D$ is close to the identity matrix. 
\begin{lem}
	Suppose that the linear map $A$ satisfies $(1-\delta,1+\delta)$ RIP for some $\delta \in (0,1)$. 
	Then the matrix $D$ satisfies 
	\[
	(1-\delta)^2 I \preceq D \preceq (1+\delta)^2 I.
	\]
\end{lem}
\begin{proof}
	Indeed, since $D$ is diagonal, we only need to show  $\frac{1}{\ncons}a_i^\top a_i\in [(1-\delta)^2,(1+\delta)^2]$ for each index $i$. Note that 
	$\frac{1}{\ncons}\twonorm{Ae_i}^2 = \frac{1}{\ncons} a_i^\top a_i$. Since $e_i$ is a sparse vector with only one nonzero, 
	using the $(1-\delta,1+\delta)$ RIP, we have $\frac{1}{\ncons} a_i^\top a_i = \frac{1}{\ncons}\twonorm{Ae_i}^2 
	\in [(1-\delta)^2 \twonorm{e_i}, (1+\delta)^2 \twonorm{e_i}^2 ] =  [(1-\delta)^2 , (1+\delta)^2 ] $ and our proof is complete.
\end{proof}

The next lemma shows that the following optimization problem is equivalent to 
the optimization problem defining flat solutions  \eqref{def: flat_solution_sparse_vector_recovery}.
\begin{equation}\label{eq: sparse_Perturbed_L_1_norm}
	\min_{x\in \RR^{d}}\; \sum_{i=1}^d |D_{ii}| |x _i|^{2-\frac{2}{k}} \quad \text{s.t.}\quad 
Ax =b.
\end{equation}

Denote by $v_{h,j}$  the $j$-th component of the vector variable $v_h$ for $1\leq h\leq k$.
\begin{lemma}
	Problem \eqref{def: flat_solution_sparse_vector_recovery} is equivalent to Problem \eqref{eq: sparse_Perturbed_L_1_norm} 
	in the following sense: 
	\begin{itemize}
		\item If $x$ solves \eqref{eq: sparse_Perturbed_L_1_norm}, then 
		any $v_i$ satisfying $x = v_1\odot \dots \odot v_k$ and $|v_{1,j}| = \dots =|v_{k,j}|$ for any 
		$1\leq j\leq \dm$ solves \eqref{def: flat_solution_sparse_vector_recovery}.
		\item If $v_1,\dots, v_k$ solves \eqref{def: flat_solution_sparse_vector_recovery}, then 
		$x = v_1\odot \dots \odot v_k$ solves \eqref{eq: sparse_Perturbed_L_1_norm}.
	\end{itemize}
\end{lemma}
\begin{proof}
	According to \eqref{eqn: trace}, the trace of the Hessian is 
	\begin{equation}
		\begin{aligned}
			\tr(D^2 f([v_i]_{i=1}^k)) &  = \sum_{i=1}^k \twonorm{\sqrt{D} (v_1\odot \dots v_{i-1}\odot v_{i+1}\dots \odot v_k)}^2 \\
				&  = \sum_{i=1}^k \sum_{j=1}^\dm D_{jj} \prod_{ h\not = i} v_{h,j}^2\\
				& = \sum_{j=1}^\dm D_{jj}  \sum_{i=1}^k  \prod_{ h\not = i} v_{h,j}^2\\
				& \overset{(a)}{\geq}  \sum_{j=1}^\dm D_{jj} k \left( \prod_{ i=1}^k v_{i,j}^{2(k-1)} \right)^{\frac{1}{k}}\\
				& =  \sum_{j=1}^\dm D_{jj} \left( \prod_{ i=1}^k |v_{i,j}| \right)^{2-\frac{2}{k}}.\\
		\end{aligned}
	\end{equation}
In the step $(a)$, 
we use the well-known AM-GM inequality. The equality holds if and only if $|v_{1,j}| = \dots =|v_{k,j}|$ for any 
$1\leq j\leq \dm$. The rest 
follows by letting $x = v_1\odot \dots \odot v_k$.

\end{proof}

For the case $k=2$, the objective is $\sum_{i=1}^d |D_{ii}| |x_i| = \|Dx\|_1$ which is the rescaled 
$\ell_1$ norm for a near-identity matrix $D$. Hence, an argument similar to those in Section \ref{sec:RIP} reveals the minimizer is uniquely $\trux$. 
We state this result formally below. 

\begin{lem}
	There is a universal constant $c>0$ such that if the linear map $A$ satisfying $(1-\delta,1+\delta)$ RIP 
	with $0<\delta <c$. Then for $k=2$, any solution $(v_1,v_2)$ to \eqref{def: flat_solution_sparse_vector_recovery}
	satisfies $\trux = v_1\odot v_2$.
\end{lem}

On the other hand, higher values of $k$ \emph{do not} encourage sparsity. In the extreme case $k\rightarrow \infty$, the objective function in \eqref{eq: sparse_Perturbed_L_1_norm} is close to $\twonorm{x}^2$ which should give a dense solution in general. Indeed, in Figure \ref{fig: depthEffec}, we plot the solution performance of \eqref{eq: sparse_Perturbed_L_1_norm} for different values $k =\{2,3,\dots,10\}$ and $\trur = \{1,2,3,4,5\}$ measured by the relative error $\frac{\twonorm{x-\trux}}{\twonorm{\trux}}$. \footnote{We set $d=1000$ and $m = 3 \lceil \trur \log \dm \rceil$ and generate the signal $\trux$ with first $k$ components being $1$ and zero otherwise. For each configuration of $(k,\trur)$, we randomly generate $25$ realizations of the Gaussian sensing matrix $A$ and solve \eqref{eq: sparse_Perturbed_L_1_norm} for each $A$. The performance metric  $\frac{\twonorm{x-\trux}}{\twonorm{\trux}}$ is averaged over these $25$ trials.} Indeed, exact recovery is observed for $k=2$, while the relative error degrades significantly as   $k$ increases. 

\begin{figure}[h]
	\centering
	\includegraphics[width= 0.75 \textwidth]{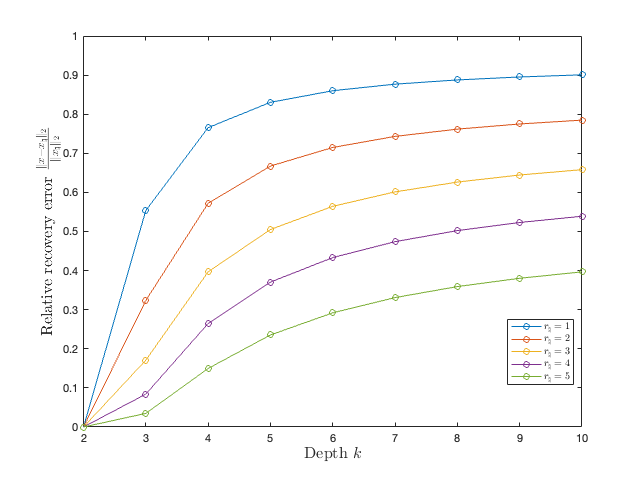}
	\caption{The effect of depth for different choice of sparsity $\trur$}\label{fig: depthEffec}
\end{figure}

%% file: extension_to_noisy_observation.tex
\section{Extension to noisy observation}\label{sec: noisyobservation}
This section considers an extension of the flat solution concept to the setting where the observations are corrupted by noise: 
\begin{equation}
	b = \Amap(\truM) + e, \quad e\sim N(0,\sigma ^2I_m),
\end{equation}
where $\sigma>0$ is the noise level and $N(0,I_m)$ is the  standard $m$-dimensional Gaussian. Our discussion in the rest of the paper focused on the simpler case $\sigma= 0$. 
We define the flat solution in this setting as follows. We continue to use the scaled trace $\str(D^2 f(L,R))$ as the flatness measure of the objective function. However, instead of considering all solutions $(L,R)$ that interpolate the data, we consider those pairs $(L,R)$ that are in the sublevel set: 
\begin{equation}\label{eq: noiseLevelSubset}
\{(L,R)\mid \twonorm{\Amap(LR^\top)-b}\leq \twonorm{e}\}.
\end{equation} 
The reason for this choice is that 
in the noisy observation setting, the global solution of \eqref{eq: main minimization} (with $k=\min\{\dmone,\dmtwo\}$) has the potential of overfitting no matter what regularization has been enforced. Indeed, consider the simplest case $\Amap = \mathcal{I}$, i.e.,
the map $\Amap$ is the identity map. In this setting, any global minimizer  $LR^\top$ is simply the observation $b = \truM +e$ itself. 
With the above preparation, we define the flat solutions to be the minimizers of the following problem.
\begin{equation}\label{def: noisyflatsolution}
	\min_{L\in \RR^{d_1\times k}, \; R\in \RR^{d_2\times k}} \str(D^2 f(L,R))\quad \text{subject to}\quad  \twonorm{\Amap(LR^\top)-b}\leq \twonorm{e}.
\end{equation}

The goal of the section is to prove the following.

\begin{thm}[Noisy matrix sensing]\label{thm: thmMS}
	Suppose that  $\mathcal{A}$ is  a Gaussian ensemble and the noise follows $e\sim N(0,\sigma^2 I_m)$. Then there exists universal constants $c,C$ such that for any $m \geq C\trur d_{\max}$, with probability at least $1-C\exp(-c(\dmone+\dmtwo))$, any solution  $(L_f,R_f)$ of \eqref{def: noisyflatsolution} satisfies 
	\begin{equation}\label{eq: mainerrorboundInformal}
		\fronorm{L_f R^\top_f -\truM}\leq C
		\sigma\sqrt{\frac{ \trur (\dmone +\dmtwo)}{m}}.
	\end{equation}
\end{thm}
Note that the bound $\sigma\sqrt{\frac{ \trur (\dmone +\dmtwo)}{m}}$  is minimax optimal according to \cite{candes2011tight}.

\subsection{Proof of Theorem \ref{thm: thmMS}}
Following \eqref{eq: scaled trace Sum of Square} and \eqref{eq: PerturbnucnormX} in Section \ref{sec: A_convex_relaxation_for_flat_solutions}, we see that  \eqref{def: noisyflatsolution} is equivalent to (in the sense of Theorem \ref{thm: equivalenceConvexNonconvexScaledTraceHessioan})
minimizing the nuclear norm over rank constrained matrices so long as $D_1,D_2$ matrices are invertible\footnote{
	Note that the condition
	$\Amap(LR^\top) = b$ is not needed  for \eqref{eqn:scaled_trace_general} to hold, which is critical for the step  \eqref{eq: scaled trace Sum of Square} to hold in the noisy case.
}:
\begin{equation}\label{eq: PerturbnucnormXnoisy}
	\min_{X\in \RR^{d_1\times d_2}:~\rank(X)\leq k} \nucnorm{X}\qquad\textrm{subject to}\qquad \norm{\Amap(D_1^{-1}XD_2^{-1})-b}\leq \norm{e}. 
\end{equation}
Let $\hat{Y}$ be any minimizer of \eqref{eq: PerturbnucnormXnoisy}. Also denote the scaled linear map $\tilde{\Amap} (\cdot)= \Amap(D_1^{-1} \cdot D_2^{-1})$, the scaled ground truth $Y_0 = D_1\truM D_2$, 
and the difference $\Delta = \hat{Y} - Y_0$. 
Our task is to show  $\fronorm{\Delta}\lesssim 
\sigma \sqrt{\frac{\trur(\dmone +\dmtwo}{m}}$. 
The bound \eqref{eq: mainerrorboundInformal} then immediately follows using $D_1LR^\top D_2 = X$ and the fact that $D_1$ and $D_2$ are near identity from Lemma \ref{lem:cond_num}.

\subsubsection{Bound $\fronorm{\Delta}$} Our proof is based on the argument in \cite{recht2010guaranteed}.  Starting with the feasibility of $\hat{Y}$, we have 
\begin{equation}\label{eq: feasibilityBound}
	\begin{aligned}
		& \frac{1}{2m}\twonorm{\tilde{\Amap}(\hat{Y})-b}^2  \leq \frac{1}{2m}\twonorm{e}^2 
		\overset{(a)}
		{\iff} 
		 \frac{1}{2m}\twonorm{\tilde{\Amap}(\Delta)}^2 +\frac{1}{m} \inprod{\tilde{\Amap}(\Delta)}{e}\leq 0\\ 
		\implies 
		& \frac{1}{2m}\twonorm{\tilde{\Amap}(\Delta)}^2 \leq 
		-\frac{1}{m} \inprod{\Delta}{\tilde{\Amap}^*e} 
		\overset{(b)}{\leq} 
		\opnorm{\tilde{\Amap}^*e}\nucnorm{\Delta}.
	\end{aligned}
\end{equation}
Here the step $(a)$ is due to expanding the square for the left hand side, and the step $(b)$ is due to H\"{o}lder's inequality. 
We next try to lower bound $\frac{1}{2m}\twonorm{\tilde{\Amap}(\Delta)}^2 $ and upper bound 	$\opnorm{\tilde{\Amap}^*e}$ and $\nucnorm{\Delta}$.

\paragraph{Upper bound $\nucnorm{\Delta}$} First, let us introduce a lemma that decomposes $\Delta$. 
\begin{lem}\cite[Lemma 2.3 and 3.4]{recht2010guaranteed}\label{lem: AB}
	For any $A,B\in \real^{\dm\times \dm}$, there exists $B_1,B_2$ such that (1) $B= B_1+B_2$, (2) $\rank(B_1)\leq \rank(A)$, (3) $AB^\top_2 =0$ and $A^\top B_2 =0$, (4) $\nucnorm{A+B_2} = \nucnorm{A}+\nucnorm{B_2}$, and (5) $\inprod{B_1}{B_2}=0$.
\end{lem}
Using Lemma \ref{lem: AB}, we can decompose $\Delta =R_0 + R_c$ such that $Y_0 R_c^\top = 0, Y_0^\top R_c =0$, $R_0\leq 2\trur$, $\inprod{R_0}{R_c}=0$, and $\nucnorm{Y_0 + R_c} = \nucnorm{Y_0} + \nucnorm{R_c}$. Hence, we have 
\begin{equation}\label{eq: Rinequality}
	\nucnorm{\hat{Y}}= \nucnorm{Y_0 + \Delta} \geq  \nucnorm{Y_0+R_c} - \nucnorm{R_0} =  \nucnorm{Y_0}+\nucnorm{R_c} - \nucnorm{R_0}.
\end{equation}
Using the optimality of $\nucnorm{\hat{Y}}\leq \nucnorm{Y_0}$, we have that 
\begin{equation}\label{eq: Rinequality2}
	\nucnorm{R_c} \leq \nucnorm{R_0}.
\end{equation}
Using the fact that $R_0$ has rank no more than $2\trur$ and $\inprod{R_0}{R_c} =0$, we have 
\begin{equation} \label{eq: UboundDelta}
\nucnorm{\Delta}\leq \nucnorm{R_c} + \nucnorm{R_0}\leq 
2\nucnorm{R_0}\overset{(b)}{\leq} 2\sqrt{2\trur}\fronorm{R_0} \overset{(c)}{\leq} 2\sqrt{2\trur}\fronorm{\Delta}. 
\end{equation} 
Here in the first step, we use the triangle inequality for $\nucnorm{\Delta}$. This finishes the upper bound of $\nucnorm{\Delta}$.

\paragraph{Lower bound on$\frac{1}{m} \twonorm{\tilde{\Amap}(\Delta)}^2$.} Next we partition $R_c$ into a sum of matrices $R_1, R_2, \ldots$ 
each of rank at most $3\trur$ as in 
\cite[Theorem 3.3]{recht2010guaranteed}.  Let $R_c= U\diag(\sigma)V'$ be the
singular value decomposition of $R_c$.  For each $i\geq 1$ define
the index set $I_i = \{3\trur (i-1)+1,\ldots,3\trur i\}$, and let
$R_i:=U_{I_i}\diag(\sigma_{I_i})V_{I_i}'$. Using the fact that $\inprod{R_c}{R_0}=0$ and the construction of $R_i$, $i\geq 1$, we also have 
\begin{equation}\label{eq: Riorthogonal}
	\langle R_i,
	R_j \rangle = 0,\quad \forall \,i \not = j, i,j\geq 0.
\end{equation}
By construction, we have
\begin{equation}
	\sigma_k \leq \frac{1}{3\trur} \sum_{j\in I_i} \sigma_j
	\qquad \forall\, j\in I_{i+1},
\end{equation}
which implies $\|R_{i+1}\|_F^2 \leq \frac{1}{3\trur} \|R_i\|_*^2$.  We
can then compute the following bound
\begin{equation}
	\begin{aligned}\label{eq: boundIntermediateTerm}
		\sum_{j\geq 2} \|R_j\|_F 
		\leq \frac{1}{\sqrt{3\trur}} \sum_{j\geq 1}
		\|R_j\|_* 
		= \frac{1}{\sqrt{3\trur}} \|R_c\|_* 
		\overset{(a)}{\leq} \frac{1}{\sqrt{3\trur}}
		\|R_0\|_* 
		\overset{(b)}{\leq} \frac{\sqrt{2\trur}}{\sqrt{3\trur}}
		\fronorm{R_0} \\ 
	\end{aligned} 
\end{equation}
where the step $(a)$ is due to \eqref{eq: Rinequality}, and the step $(b)$ is due to the fact
that $\rank(R_0) \leq 2\trur$. From this inequality, we also have 
\begin{equation}\label{eq: deltaFronorm}
	\fronorm{\Delta}\leq \fronorm{R_0+R_1} +\sum_{j\geq 2} \fronorm{R_j} 
	\leq \fronorm{R_0+R_1} + 2\fronorm{R_0} \leq 3\fronorm{R_0+R_1}
\end{equation}
The last equality is due to that $\inprod{R_0}{R_1}=0$.
Hence, we have that
\begin{equation}\label{eq:l1rec}
	\begin{aligned}
		\frac{1}{\sqrt{m}}\|\tilde{\Amap}(\Delta)\|_2 
		\overset{(a)}{\geq}& \frac{1}{\sqrt{m}}\|\tilde{\Amap}(R_0+R_1)\|_2 -
		\sum_{j\geq 2} \frac{1}{\sqrt{m}}\|\tilde{\Amap}(R_j)\|_2 \\
		\geq & (1-\delta) \fronorm{R_0+R_1} -
		(1+\delta) \, \sum_{j\geq 2} \fronorm{R_j}\\
		\overset{(b)}{\geq} & \left(1-\delta\right) \fronorm{R_0+R_1} - \sqrt{\tfrac{2}{3}}(1+\delta)
		\fronorm{R_0} \\
		\overset{(c)}{\geq} & c_1 \fronorm{R_0+R_1}\\ 
		\overset{(d)}{\geq} & c_2 \fronorm{\Delta}.
	\end{aligned}
\end{equation}
Here, in the step $(a)$, we use the reversed triangle inequality. In the step $(b)$, we use \eqref{eq: boundIntermediateTerm}. In the step $(c)$, we use $\inprod{R_0}{R_1}=0$ and  the choice of $\delta$. The last step $(d)$ is due to \eqref{eq: deltaFronorm}. This finishes the lower bound on $	\frac{1}{\sqrt{m}}\|\tilde{\Amap}(\Delta)\|_2 ^2$.

\paragraph{Estimating $\opnorm{\tilde{\Amap}^*e}$.} Note that $\tilde{\Amap}^*(e)= D_1^{-1}( \Amap^* e )D_2^{-2}$. Hence $\opnorm{\tilde{\Amap}^*(e)}= \opnorm{D_1^{-1} \Amap^* (e) D_2^{-2}} \leq \opnorm{D_1^{-1}} \opnorm{\Amap^* e}\opnorm{D_2^{-1}}$. From \cite[Proof of Corollary 10.10]{wainwright2019high}, we know that with probability at least $1-c\exp(-c(\dmone + \dmtwo))$, the inequality $\frac{1}{m}\opnorm{\Amap^* e}\leq C\sigma \sqrt{\frac{\dmone+\dmtwo}{m}}$ holds. Since $0.9I\preceq D_i\preceq 1.1I$, we have
\begin{equation} \label{eq: noiseOpBound}
\frac{1}{m}\opnorm{\tilde{A}(e)}\leq C'\sigma \sqrt{\frac{\dmone+\dmtwo}{m}}.
\end{equation} 

Combining  \eqref{eq: feasibilityBound}, \eqref{eq: UboundDelta}, \eqref{eq:l1rec}, and \eqref{eq: noiseOpBound}, we conclude  $\fronorm{\Delta}\lesssim 
\sigma \sqrt{\frac{\trur(\dmone +\dmtwo}{m}}$, as claimed.

\subsection{A numerical demonstration}
Finally, we validate Theorem \ref{thm: thmMS} via a numerical experiments. We compare the performance of the minimizer $\hat{X}$ of 
Problem \eqref{eq: PerturbnucnormXnoisy} for the case $k = \min\{\dmone ,\dmtwo\}$ and the solution $\hat{X}_{\text{nuc}}$ of the 
nuclear norm minimization (Problem \eqref{eq: PerturbnucnormXnoisy} with $D_1,D_2$ being the identity). 

\paragraph{Experiment setup} We set $d = \dmone = \dmtwo = 25$ and $m = 1000$. We generate the underlying unit Frobenius norm ground truth matrix $\truM$ randomly with rank $\trur =\{1,2,3,\dots,10\}$. We 
vary the noise level $\sigma = \{0.1,0.2,\dots, 1.3,1.4,1.5\}$. For each rank $\trur$, we generate the sensing Gaussian 
ensemble $\Amap$ with $m=1000$ and use the same one for different noise levels. 
Then for each noise level, we generate 25 
realization of the noise $e$ following $N(0,\sigma^2 I)$, and solve the corresponding Problem \eqref{eq: PerturbnucnormXnoisy} 
and the nuclear norm minimization problem. We then  average the error 
$\fronorm{D_1^{-1}\hat{X}D_2^{-1}-\truM}$ and $\fronorm{\hat{X}-\truM}$ over the $25$ trials for each 
configuration of $\trur$ and $\sigma$. 

\paragraph{Recovery Performance} We plot the error in Figure \ref{fig: noisyCase}. The white color indicates 
small error and the dark color indicates large error. It can be seen that it is hard to differentiate the performance 
of the solution to the nuclear norm minimization problem and the solution to Problem \eqref{eq: PerturbnucnormXnoisy}. This 
result validates our theoretical results in Theorem \ref{thm: thmMS}.

\begin{figure}[h]
		\centering
	\begin{subfigure}{0.48\textwidth}
			\centering
			\includegraphics[width=\textwidth]{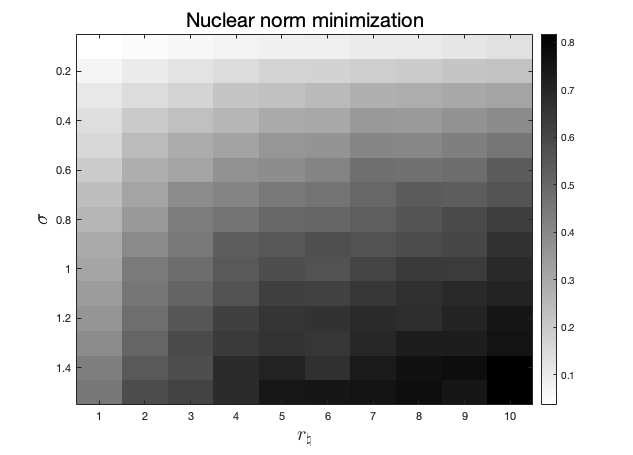}
		\end{subfigure}
	\begin{subfigure}{0.45 \textwidth}
			\includegraphics[width= \textwidth]{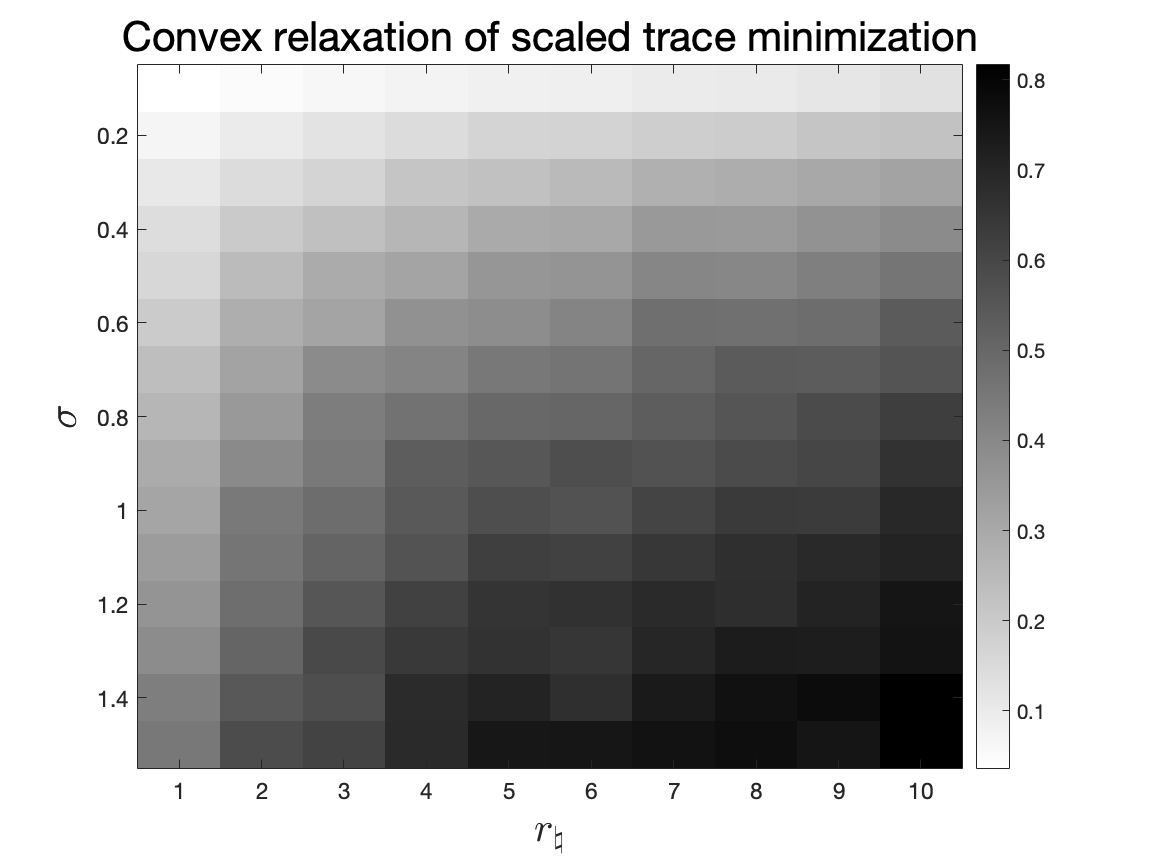}
		\end{subfigure}
	\caption{Matrix Sensing with nuclear norm minimization (left) and the convex 
		relaxation (or $k=d$) of Problem \eqref{eq: PerturbnucnormXnoisy} (right) for 
		different configurations of $(\trur,\sigma)$.
		}\label{fig: noisyCase}
\end{figure}